\documentclass{article}

\usepackage{arxiv}

\usepackage[utf8]{inputenc} % allow utf-8 input
\usepackage[T1]{fontenc}    % use 8-bit T1 fonts
\usepackage{hyperref}       % hyperlinks
\usepackage{url}            % simple URL typesetting
\usepackage{booktabs}       % professional-quality tables
\usepackage{amsfonts}       % blackboard math symbols
\usepackage{nicefrac}       % compact symbols for 1/2, etc.
\usepackage{microtype}      % microtypography
\usepackage{lipsum}		% Can be removed after putting your text content
\usepackage{graphicx}
\usepackage{natbib}
\usepackage{doi}

\usepackage{amsmath}
\usepackage[ruled]{algorithm2e}
\usepackage{mathtools}
\usepackage{booktabs,makecell,multirow}
\usepackage{algorithmic}

\newtheorem{theorem}{Theorem}
\newtheorem{definition}{Definition}

\newtheorem{proof}{Proof}

\title{Blood Glucose Control Via Pre-trained Counterfactual Invertible Neural Networks}

\author{ \href{https://orcid.org/0000-0003-2167-4082}{\includegraphics[scale=0.06]{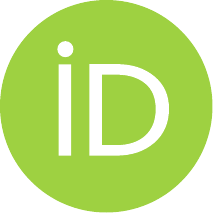}\hspace{1mm}Jingchi Jiang}\\
The Artificial Intelligence Institute\\ Harbin Institute of Technology\\ Harbin, Heilongjiang, China\\
	\texttt{jiangjingchi@hit.edu.cn} \\
	\And
	\href{https://orcid.org/0000-0002-6100-8888}{\includegraphics[scale=0.06]{orcid.pdf}\hspace{1mm}Rujia Shen} \\
	Faculty of Computing\\ Harbin Institute of Technology\\ Harbin, Heilongjiang, China\\
	\texttt{shenrujia@stu.hit.edu.cn} \\
	\And
	Boran Wang \\
	Faculty of Computing\\ Harbin Institute of Technology\\ Shenzhen, Guangdong, China\\
	\texttt{wangboran@hit.edu.cn} \\
	\And
	Yi Guan \\
	Faculty of Computing\\ Harbin Institute of Technology\\ Harbin, Heilongjiang, China
	\texttt{guanyi@hit.edu.cn}
}

\renewcommand{\shorttitle}{\textit{arXiv} Template}

\hypersetup{
pdftitle={A template for the arxiv style},
pdfsubject={q-bio.NC, q-bio.QM},
pdfauthor={David S.~Hippocampus, Elias D.~Striatum},
pdfkeywords={First keyword, Second keyword, More},
}

\begin{document}
\maketitle

\begin{abstract}
	Type 1 diabetes mellitus (T1D) is characterized by insulin deficiency and blood glucose (BG) control issues. The state-of-the-art solution for continuous BG control is reinforcement learning (RL), where an agent can dynamically adjust exogenous insulin doses in time to maintain BG levels within the target range. However, due to the lack of action guidance, the agent often needs to learn from randomized trials to understand misleading correlations between exogenous insulin doses and BG levels, which can lead to instability and unsafety. To address these challenges, we propose an introspective RL based on Counterfactual Invertible Neural Networks (CINN). We use the pre-trained CINN as a frozen introspective block of the RL agent, which integrates forward prediction and counterfactual inference to guide the policy updates, promoting more stable and safer BG control. Constructed based on interpretable causal order, CINN employs bidirectional encoders with affine coupling layers to ensure invertibility while using orthogonal weight normalization to enhance the trainability, thereby ensuring the bidirectional differentiability of network parameters. We experimentally validate the accuracy and generalization ability of the pre-trained CINN in BG prediction and counterfactual inference for action. Furthermore, our experimental results highlight the effectiveness of pre-trained CINN in guiding RL policy updates for more accurate and safer BG control.
\end{abstract}

% keywords can be removed
\keywords{Blood glucose control \and Invertible neural networks \and Forward prediction \and Counterfactual inference \and Reinforcement learning \and Pre-trained neural networks}

\section{Introduction}\label{sec:introduction}
Type 1 diabetes (T1D) stands as a chronic autoimmune disease stemming from the progressive destruction of beta cells in the pancreas, which impairs the body's ability to produce endogenous insulin \citep{cappon2023individualized}. Individuals afflicted with T1D necessitate exogenous insulin administration to regulate blood glucose (BG) levels within the specified target range of 70–180 mg/dl, thereby avoiding the risk of hypoglycemia ($<$70 mg/dl) or hyperglycemia ($>$180 mg/dl) \citep{li2019glunet}. If the BG concentration persistently and frequently exceeds the hyperglycemic threshold, patients may experience a range of severe cardiovascular complications, in addition to nephropathy and neuropathy \citep{nathan2014diabetes}. Consequently, the administration of exogenous insulin several times a day becomes imperative to mitigate elevated BG levels. Nevertheless, the excessive dosing of exogenous insulin can precipitate hypoglycemia, characterized by BG levels falling below 70 mg/dL, a dangerous condition even in the short term due to its potential to induce fainting, dizziness, coma, and ultimately, fatality \citep{cryer2003hypoglycemia}.

Most questions that arise in both daily life and scientific research are fundamentally predictive or introspective in nature. For example, predictive work may seek to answer questions such as, ``If patients with T1D adjust their exogenous insulin doses under normal carbohydrate intake, how will their BG levels change?'' \citep{li2019glunet,lee2023glucose}. Conversely, introspective research may delve deeper, posing hypothetical questions such as, ``What amount of exogenous insulin dose is required for patients with T1D to maintain BG levels within the target range?'' These questions are challenging to address through statistical analysis alone and require a causality-based method \citep{sun2023causality,liu2024cignn} that examines the relationships between external interventions (such as exogenous insulin dose and carbohydrate intake) and a hypothetical counterfactual scenario. Causality, as the underlying representation, typically remains invariant despite fluctuations in patient conditions \citep{lu2021invariant,wang2023causal}. However, despite the success of deep learning in forward prediction \citep{farnoosh2021deep,cao2023inparformer,zeng2023transformers}, most models rely on training and testing datasets that are independent and identically distributed (IID). This limitation leads to considerable prediction instability when faced with out-of-distribution (OOD) generalization \citep{d2022underspecification}. Furthermore, the inability of current correlation-driven approaches to distinguish between causal relationships and misleading correlations underscores their limited adaptability \citep{liu2023neuron}. Reinforcement learning (RL) agents based on misleading correlations often fail to control BG levels for patients with T1D and show instability and unsafety. Surprisingly, causal relationships can exhibit invariance across analogous yet distinct datasets \citep{edmonds2020theory}, suggesting that integrating causal relationships between features may offer solutions to address these challenges \citep{wang2022causal}.

Counterfactual inference \citep{pearl2010causal,kuang2020causal}, akin to human introspection, constitutes a pivotal component of the ladder of causation. The forward prediction allows us to anticipate the consequences of specific actions, while introspection empowers us to contemplate the way to specific outcomes hypothetically. Although introspection cannot alter existing factual circumstances, it can inform future behavior by leveraging prior causal relationships derived from historical interactions. As previously noted, forward prediction and counterfactual inference are distinct cognitive processes, but both rely on a deep understanding of causal relationships \citep{hofler2005causal,morgan2015counterfactuals}. 

In this paper, we aim to integrate forward prediction and counterfactual reasoning while leveraging causal information to address the challenges of generalization in BG level prediction. Once the model can effectively generalize, we believe it can fully perceive the underlying causality. Moreover, we should ultimately ensure the safety and accuracy of BG control. Our contributions are three-fold:

\begin{itemize}
    \item We propose Counterfactual Invertible Neural Networks (CINN), a bidirectional neural network that integrates forward prediction and counterfactual inference. CINN is constructed based on the topological order of causal relationships, enabling interpretable and robust solutions for generalizing OOD scenarios in BG level prediction.
    \item We integrate the pre-trained CINN into RL and utilize CINN as an introspective block within the agent to guide policy updates. This innovative approach involves engaging in counterfactual inference prior to the execution of each agent action, thereby reducing futile trials and errors during the training procedure, leading to more stable and safer BG control.
    \item Experimental results demonstrate the superior generalization ability of CINN in both IID and OOD scenarios for BG prediction and counterfactual inference. Furthermore, by incorporating the capability of counterfactual inference, RL with pre-trained CINN achieves state-of-the-art performance in terms of control effectiveness, consistently maintaining BG levels within the target range.
\end{itemize}

\section{Preliminaries}

In Section \ref{Causal_Inference}, we first discuss the relationship between forward prediction and counterfactual inference, which forms the probabilistic logical foundation of CINN. Subsequently, in Section \ref{INN} and Section \ref{OWN}, we introduce the concepts of invertible neural networks and orthogonal weight normalization, respectively. These concepts guide the design of the network architecture and the training process of CINN.

\subsection{Forward Prediction and Counterfactual Inference\label{Causal_Inference}}

The causal relationships between the physiological state of patients with T1D and interventions can be represented as a directed acyclic graph (DAG) \citep{liu2019learning,lipsky2022causal} denoted as $\mathcal{G}=(\mathcal{X}, \mathcal{E})$, with three types of temporal nodes, where $\mathcal{X}=\{S^t \in\mathbb{R}^n, S^{t+1} \in\mathbb{R}^n, A^t \in\mathbb{R}^2\}$. Here, $A^t\in\mathbb{R}^2$ represents the exogenous insulin dose and the carbohydrate intake. The joint distribution $p(S^t, S^{t+1}, A^t)$ is factorized as the product of conditional distributions given the parents $\text{Pa}(\cdot)$ of each node, i.e., $p(\mathcal{X}) = p(S^t|\text{Pa}(S^t))p(A^t|\text{Pa}(A^t))p(S^{t+1}|\text{Pa}(S^{t+1}))$, where $\text{Pa}(S^t)=\empty, \text{Pa}(S^{t+1})=\{S^t,A^t\},\text{Pa}(A^t)=\{S^t\}$. Each edge in $\mathcal{E}$ signifies a causal relationship: a path from $S^t$ to $A^t$ implies that $S^t$ is a potential cause of $A^t$. The causal effect of $S^t$ on $A^t$ is captured by the conditional distribution $p(A^t | S^t)$.

\begin{figure}[t]
    \centering
    \includegraphics[width=0.7\columnwidth]{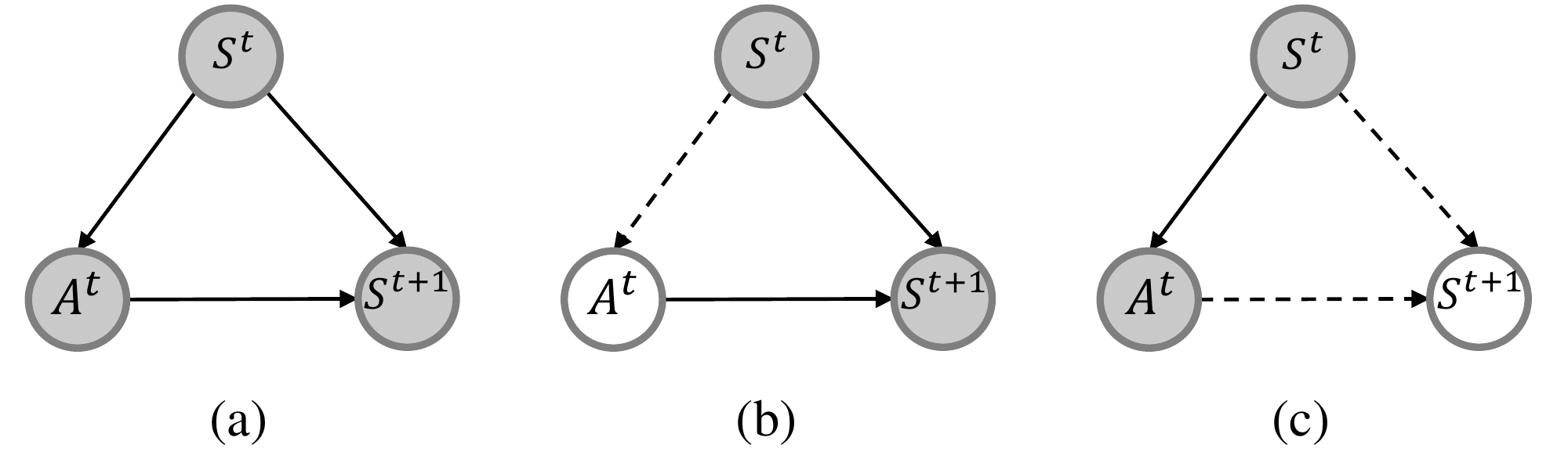} % Reduce the figure size so that it is slightly narrower than the column. Don't use precise values for figure width.This setup will avoid overfull boxes.
    \caption{Causal graphs illustrating the causal relationships between the states $S^t$ and $S^{t+1}$ of patients with T1D at time $t$ and $t+1$, and the interventions $A^t$ at time $t$, are delineated as follows: (a) the causal graph $\mathcal{G}$ without external intervention; (b) the causal graph $\mathcal{G}_{do(A^t=a^t)}$ with external intervention on $A^t$; (c) the counterfactual graph $\tilde{\mathcal{G}}_{do(S^{t+1}=\tilde{s}^{t+1})}$ in counterfactual inference with $S^{t+1}=\tilde{s}^{t+1}$ as the condition.}
    \label{fig1}
\end{figure}

An illustrative example of the causal graph $\mathcal{G}$ is presented in Figure \ref{fig1} (a), where $S^t$ represents the physiological state of patients with T1D at time $t$, $A^t$ signifies the exogenous insulin dose and the carbohydrate intake at time $t$, and $S^{t+1}$ denotes the state at the subsequent time $t+1$. Since the state $S^t$ determines the action $A^t$, a causal relationship $S^t\rightarrow A^t$ is established. The causal effect from $A^t$ to $S^{t+1}$ is captured by the conditional distribution restricted to the path $A^t\rightarrow S^{t+1}$. In prevalent correlation-based forward prediction models \citep{farnoosh2021deep,cao2023inparformer,zeng2023transformers}, it is expected to predict the future physiological state $S^{t+1}$ based on the current state $S^{t}$ along with exogenous insulin dose and carbohydrate intake $A^t$, ignoring the confounding condition. This omission establishes another backdoor pathway between $A^t$ and $S^{t+1}$, potentially leading to a flawed assessment of the causal effect between $A^t$ and $S^{t+1}$ \citep{greenland2001confounding,mcnamee2003confounding}. Subsequently, we delineate the probability conditional distributions of forward prediction and counterfactual inference to elucidate their relationship.

\subsubsection{Forward Prediction} 

In the intervened graph $\mathcal{G}_{do(A^t=a^t)}$, as illustrated in Figure \ref{fig1} (b), the causal effect of $do(A^t=a^t)$ can be formalized through the conditional distribution $p(S^{t+1}|do(A^t=a^t))$. The do-calculus provides us with a methodology to compute $p(S^{t+1}|do(A^t=a^t))$ utilizing observations from the original causal graph $\mathcal{G}$ \citep{pearl1995causal,pearl2012calculus}. Since $A^t$ undergoes an external intervention, it ceases to be influenced by $S^t$ (the edge from $S^t$ to $A^t$ is eliminated). Consequently, the do-calculus enables us to apply probabilistic techniques for forward prediction based on the intervened graph $\mathcal{G}_{do(A^t=a^t)}$. In the following, we will introduce the pivotal concept of the backdoor criterion \citep{van2014constructing,maathuis2015generalized} in do-calculus and derive the backdoor adjustment \citep{pearl2000models}. This derivation will establish the equivalence $p(S^{t+1}|do(A^t=a^t))=\sum_{s^t}p(S^{t+1}|A^t=a^t,S^t=s^t)p(S^t=s^t)$, thereby validating our approach.

\begin{definition}[Backdoor Criterion]
    In DAG $\mathcal{G}$ comprising $A^t$, $S^t$, and $S^{t+1}$, if $S^t$ is not a descendant of $A^t$, and conditioning on $S^t$ blocks all backdoor paths connecting $A^t$ and $S^{t+1}$, $S^t$ satisfies the backdoor criterion on $(A^t, S^{t+1})$.
\end{definition}

\begin{theorem}[Backdoor Adjustment]
    If $S^t$ satisfies the backdoor criterion with respect to $(A^t, S^{t+1})$, then the causal relationship from $A^t$ to $S^{t+1}$ can be expressed as: $p(S^{t+1}|do(A^t=a^t))=\sum_{s^t} p(S^{t+1}|A^t=a^t,S^t=s^t)p(S^t=s^t)$.
\end{theorem}

\begin{proof}
    Let $p_{do}$ denote the probability distribution after the intervention. We aim to determine the distribution of $p(S^{t+1}|do(A^t = a^t))$, which is equivalent to $p_{do}(S^{t+1}|A^t=a^t)$. As shown in Figure \ref{fig1} (b), after the intervention on $A^t$, the marginal distribution of $S^t$ remains unchanged, i.e., $p_{do}(S^t=s^t)=p(S^t=s^t)$. Furthermore, we also know that after the intervention, the conditional probability of $S^{t+1}$ given $S^t$ and $A^t$ remains unchanged, i.e., $p_{do}(S^{t+1}|A^t=a^t, S^t=s^t)=p(S^{t+1}|A^t=a^t, S^t=s^t)$. Therefore, we can rewrite $p(S^{t+1}|do(A^t=a^t))$ as follows:
    
    \begin{equation}
        \begin{aligned}
        p(S^{t+1}|do(A^t=a^t))
        &=p_{do}(S^{t+1}|A^t=a^t)\\
        =&\sum_{s^t}p_{do}(S^{t+1}|A^t=a^t,S^t=s^t)p_{do}(S^t=s^t|A^t=a^t)\\
        =&\sum_{s^t}p_{do}(S^{t+1}|A^t=a^t,S^t=s^t)p_{do}(S^t=s^t)\\
        =&\sum_{s^t}p(S^{t+1}|A^t=a^t,S^t=s^t)p(S^t=s^t)
        \end{aligned}
        \label{equationapp_1}
    \end{equation}
    % We have now established that $p(S^{t+1}|do(A^t=a))=\sum_{s^t}p(S^{t+1}|A^t=a,S^t=s^t)p(S^t=s^t)$.
\end{proof}

\subsubsection{Counterfactual Inference} 

Forward prediction serves as a potent methodology for investigating predictive questions, such as assessing the potential impact of the exogenous insulin dose and carbohydrate intake on BG levels. However, it lacks the capability to address introspective inquiries concerning counterfactual questions \citep{pawlowski2020deep}, such as quantifying the exogenous insulin dose and the carbohydrate intake necessary to maintain patients' BG levels consistently within the target range. This requires counterfactual inference \citep{johansson2016learning,pawlowski2020deep}.

In a hypothetical counterfactual scenario, the quantitative description of the ideal state at the future time step, denoted as $\tilde{s}^{t+1}$, retroactively determines the exogenous insulin dose and the carbohydrate intake $A^t$ (as depicted by the causal relationship from $S^{t+1}$ to $A^t$ in Figure \ref{fig1} (c)). By intervening on the state $S^{t+1}$ to attain the ideal value $do(S^{t+1}=\tilde{s}^{t+1})$, the causal effect between $A^t$ and $S^{t+1}$ can be expressed as $p(A^t|do(S^{t+1}=\tilde{s}^{t+1}))$. In counterfactuals, leveraging the do-calculus allows us to derive $p(A^t|do(S^{t+1}=\tilde{s}^{t+1}))=\sum_{s^t}p(A^t|S^{t+1}=\tilde{s}^{t+1},S^t=s^t)p(S^t=s^t)=\sum_{s^t}p(S^{t+1}=\tilde{s}^{t+1}|A^t,S^t=s^t)p(S^t=s^t)p(A^t|S^t=s^t)$. Furthermore, the formal representation of counterfactual inference for the intervention of $S^{t+1}=\tilde{s}^{t+1}$ under the condition $A^t = a^t$ can be expressed as:

\begin{equation}
    \begin{aligned}
    p(A^{t}=a^{t}|do(S^{t+1}=\tilde{s}^{t+1}))
    =&\sum_{s^t}[p_{do}(A^{t}=a^{t}|S^{t+1}=\tilde{s}^{t+1},S^t=s^t)p(S^t=s^t)]\\
    =&\sum_{s^t}[\frac{p(A^{t}=a^{t},S^{t+1}=\tilde{s}^{t+1},S^t=s^t)}{p(S^{t+1}=\tilde{s}^{t+1},S^t=s^t)}p(S^t=s^t)]\\
    =&\sum_{s^t}[\frac{p(S^{t+1}=\tilde{s}^{t+1}|A^t=a^t,S^t=s^t)p(A^{t}=a^{t},S^t=s^t)}{p(S^{t+1}=\tilde{s}^{t+1},S^t=s^t)}p(S^t=s^t)]\\
    =&\sum_{s^t}[p(S^{t+1}=\tilde{s}^{t+1}|A^t=a^t,S^t=s^t)\frac{p(A^{t}=a^{t}|S^t=s^t)}{p(S^{t+1}=\tilde{s}^{t+1},S^t=s^t)}p(S^t=s^t)]\\
    \end{aligned}
    \label{equationapp_2}
\end{equation}

Due to the external intervention on $S^{t+1}$, we have $p(S^{t+1}=\tilde{s}^{t+1}|S^t=s^t)=p(S^{t+1}=\tilde{s}^{t+1})=1$. Additionally, $p(A^t=a^t|do(S^{t+1}=\tilde{s}^{t+1}))$ can be expressed as: $\sum_{s^t}p(S^{t+1}=\tilde{s}^{t+1}|A^t=a^t,S^t=s^t)p(A^t=a^t|S^t=s^t)p(S^t=s^t)$. Comparing forward prediction $p(S^{t+1}=s^{t+1}|do(A^t=a^t))$ and counterfactual inference $p(A^t=a^t|do(S^{t+1}=\tilde{s}^{t+1}))$, both essentially require the ability to infer the future state $S^{t+1}$ conditioned on the current action $A^t$ and the current state $S^t$, i.e., $p(S^{t+1}|A^t=a^t,S^t=s^t)$, as well as the prior marginal distribution $p(S^t = s^t)$. The difference lies in the fact that the counterfactual inference requires an additional inference capability, $p(A^t = a^t | S^t = s^t)$. Given the above similarities between forward prediction and counterfactual inference, this paper aims to investigate the integration of forward prediction and counterfactual inference within a unified framework, enabling it to encompass bidirectional capabilities.

\subsection{Invertible Neural Networks\label{INN}}

In continuous decision-making for BG control, there exists a Markovian property between $S^t$ and $S^{t+1}$. When seeking to predict $S^{t+1}$ from the known $S^t=s^t$, the prevailing theory in forward prediction typically provides a model $\hat{s}^{t+1}=F(s^t, a^t)$ to answer forward prediction questions. Conversely, when aiming to identify the optimal action to achieve the target state $\tilde{s}^{t+1}$ through introspective hypothetical inference, there arises a need to express the backward model $\hat{a}^{t}=I(s^t, \tilde{s}^{t+1})$ as a conditional probability $p(\hat{a}^{t}|s^t, \tilde{s}^{t+1})$ in order to explore the action $\hat{a}^t$ and answer counterfactual inference questions. However, the irreversibility of deep learning models poses a challenge in directly performing counterfactual inference within the same network architecture.

Surprisingly, in medical image synthesis \citep{bannister2022deep,wang2023variable}, researchers often utilize invertible neural networks (INN) to capture the transformation relationships between multi-modal medical images, which involve bidirectional inference problems. The basic structure in INN \citep{dinh2016density} contains a series of identical basic invertible blocks, with each block composed of two complementary affine coupling layers. Specifically, for any representation $u\in\mathbb{R}^d$ in the latent space, the operation in the basic invertible block involves partitioning the input into two parts, $u_1$ and $u_2$, which are subsequently transformed linearly or nonlinearly by functions $m_{1,2}$ and $n_{1,2}$. These functions can be arbitrary and may not be invertible. The final output, $\hat{v}_1$ and $\hat{v}_2$ is obtained by alternating the coupling process. Figure \ref{fig_app1} illustrates the forward and backward computation. Specifically, the forward computation formulas are:

\begin{equation}
    \begin{aligned}
    \hat{v}_1 &= u_1\odot\exp(m_2(u_2))+n_2(u_2)\\
    \hat{v}_2 &= u_2\odot\exp(m_1(\hat{v}_1)) + n_1(\hat{v}_1)\\
    \end{aligned}
    \label{equationapp_3}
\end{equation}

Given the output $\tilde{v}$, the backward computation formulas are:

\begin{equation}
    \begin{aligned}
    \hat{u}_2 &= (\tilde{v}_2 - n_1(\tilde{v}_1)) \odot \exp(-m_1(\tilde{v}_1))\\
    \hat{u}_1 &= (\tilde{v}_1 - n_2(\hat{u}_2)) \odot \exp(-m_2(\hat{u}_2))\\
    \end{aligned}
    \label{equationapp_4}
\end{equation}

where $\exp$ means bitwise e-exponential calculation and $\odot$ means the Hadamard product.During the bidirectional training of INN, we evaluate the deviation between the ground-truth $v$ and the predicted output $[\hat{v}_1, \hat{v}_2]$ in the forward prediction. This deviation is quantified as $\mathcal{L}_F(v,[\hat{v}_1, \hat{v}_2])$, where $\mathcal{L}_F$ can represent various supervised loss functions, such as mean squared error for regression tasks or cross-entropy loss for classification tasks. Conversely, for the backward process, we establish the loss function $\mathcal{L}_I$, which evaluates the discrepancy between the ground-truth input distribution $p(u)$ and the predicted inverse distribution $q(\hat{u})$. This distribution distinction can be achieved through Maximum Mean Discrepancy \citep{arbel2019maximum}.

INN undergo iterative forward and backward processes, accumulating bidirectional gradients before updating their parameters to minimize forward and backward errors during training and enhance efficiency. However, it is essential to acknowledge that INN face significant design constraints. One such constraint is the stringent requirement for matching dimensions between inputs and outputs, which may constrain their adaptability to diverse tasks. For instance, in BG control for patients with T1D, where the action space involves exogenous insulin dose and carbohydrate intake ($a^t\in\mathbb{R}^2$) and the states $s^t,s^{t+1}\in\mathbb{R}^n$ with $n>2$, simultaneous execution of $\hat{s}^{t+1}=F(s^t, a^t)$ and $\hat{a}^{t}=I(s^t, \tilde{s}^{t+1})$ becomes challenging.

\begin{figure}[t]
    \centering
    \includegraphics[width=0.6\columnwidth]{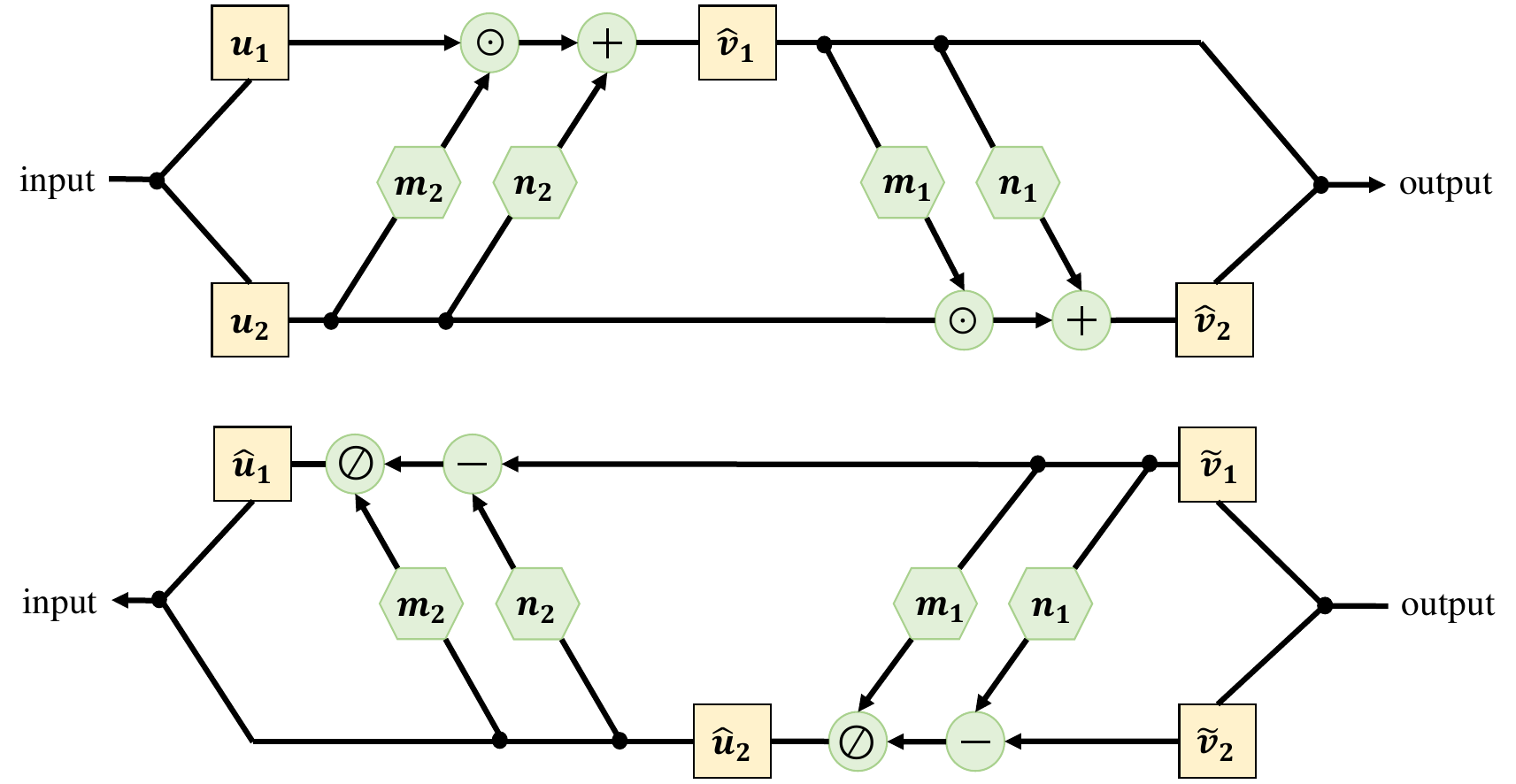} % Reduce the figure size so that it is slightly narrower than the column. Don't use precise values for figure width.This setup will avoid overfull boxes.
    \caption{Invertible Neural Networks}
    \label{fig_app1}
\end{figure}

\subsection{Orthogonal Weight Normalization\label{OWN}}

The conventional forward prediction neural networks can be viewed as a function $F(s^t,a^t;\theta)$ defined by parameters $\theta$, designed to fit the provided training datasets and effectively predict the future state $\hat{s}^{t+1}\in\mathbb{R}^n$. This function $F(s^t,a^t;\theta)$ can comprise various linear and nonlinear functions. For simplicity, we assume that it includes a transformation $\hat{s}^{t+1}=W\cdot [s^t,a^t]+B$ with the learnable weight $W\in\mathbb{R}^{n\times (n+2)}$ and the bias $B\in\mathbb{R}^n$. Thus, the learnable parameters are denoted as $\theta=\{W, B\}$. However, the backward process, expressed as $\hat{a}^t=W^{-1}(\tilde{s}^{t+1}-B) / s^t$, where $/$ denotes the truncation of $s^t$ from $W^{-1}(\tilde{s}^{t+1}-B)$, often poses challenges due to the  irreversibility of $W$.

To address the above issue and the problems of matching dimensions between inputs and outputs, an invertible technique called orthogonal weight normalization (OWN) \citep{huang2018orthogonal} has been introduced. It aims to train deep neural networks using orthogonal weight matrices $W\in\mathbb{R}^{n\times (n+2)}$ and formulates it as a constrained optimization problem:
\begin{equation}
	\begin{aligned}
	\theta^*&=\mathop{\arg\min}\limits_{\theta} \mathbb{E}_{s^t,s^{t+1},a^t}[\mathcal{L}_F(F(s^t,a^t;\theta),s^{t+1})]\ \rm{s.t.} W\in\mathcal{O}^{n\times (n+2)}
	\end{aligned}
    \label{equation3}
\end{equation}

where the matrix family $\mathcal{O}^{n\times (n+2)}=\{W\in\mathbb{R}^{n\times (n+2)}, WW^T=\mathbb{1}\}$ embodies orthogonality within weight matrices, enabling neural networks to maintain reversibility during backward process, as demonstrated by the operation $\hat{a}^t=W^{T}(\tilde{s}^{t+1}-B) / s^t$. $\mathbb{1}$ represents the unit matrix. Unlike INN, which rely on specific structural designs to achieve inverse operations, the OWN imposes reversible constraints on the weight parameters. This paper merges these two classes of invertible techniques and proposes the CINN capable of flexibly adapting to diverse BG level prediction and counterfactual inference for patients with T1D. To our knowledge, this study is the first to unify these two invertible methods into a bidirectional model and apply it to both forward prediction and counterfactual inference for stable and safer BG control.

\section{Counterfactual Invertible Neural Network \label{CINN}}

In this section, we introduce the Counterfactual Invertible Neural Network (CINN), which serves as a pre-trained and frozen introspective network to guide the RL updates discussed in the following section. The design of the CINN follows several basic principles. First, in order to improve interpretability, which is of paramount importance in the medical domain \citep{zhang2021interpretability,wang2023interpretable}, the CINN must have an architecture that humans can understand. Second, the base architecture should be simple yet adaptable to efficiently serve as a component for the RL agents. Finally, the CINN is designed to support not only forward prediction under interventions but also counterfactual inference under IID and OOD scenarios for a better understanding of causality. Note that in the real world, the causal graph is usually determined by dynamic equations. For convenience, we illustrate a causal graph with a state dimension of 6. However, in BG control, the physiological state of patients with T1D is represented by 13 dimensions \citep{man2014uva}, forming two symmetric blocks and one asymmetric block. In the following, we will explore how these principles converge within CINN.

\subsection{The Interpretable Construction of CINN\label{section32}}

Establishing the relationships between the causal graph $\mathcal{G}$ and the neural network structure is essential to improve the interpretability of the model and to mitigate the effects of misleading correlations \citep{wein2021graph,zevcevic2021relating}. As depicted in Figure \ref{fig2} (a), our CINN consists of two types of basic invertible blocks (symmetric and asymmetric). We believe that the interpretable combination of these basic blocks based on the causal graph is a prerequisite for constructing the causal effects among interventions, states, and future states.

\begin{figure*}[t]
    \begin{minipage}[t]{0.7\linewidth}
        \centering 
        \includegraphics[width=\textwidth]{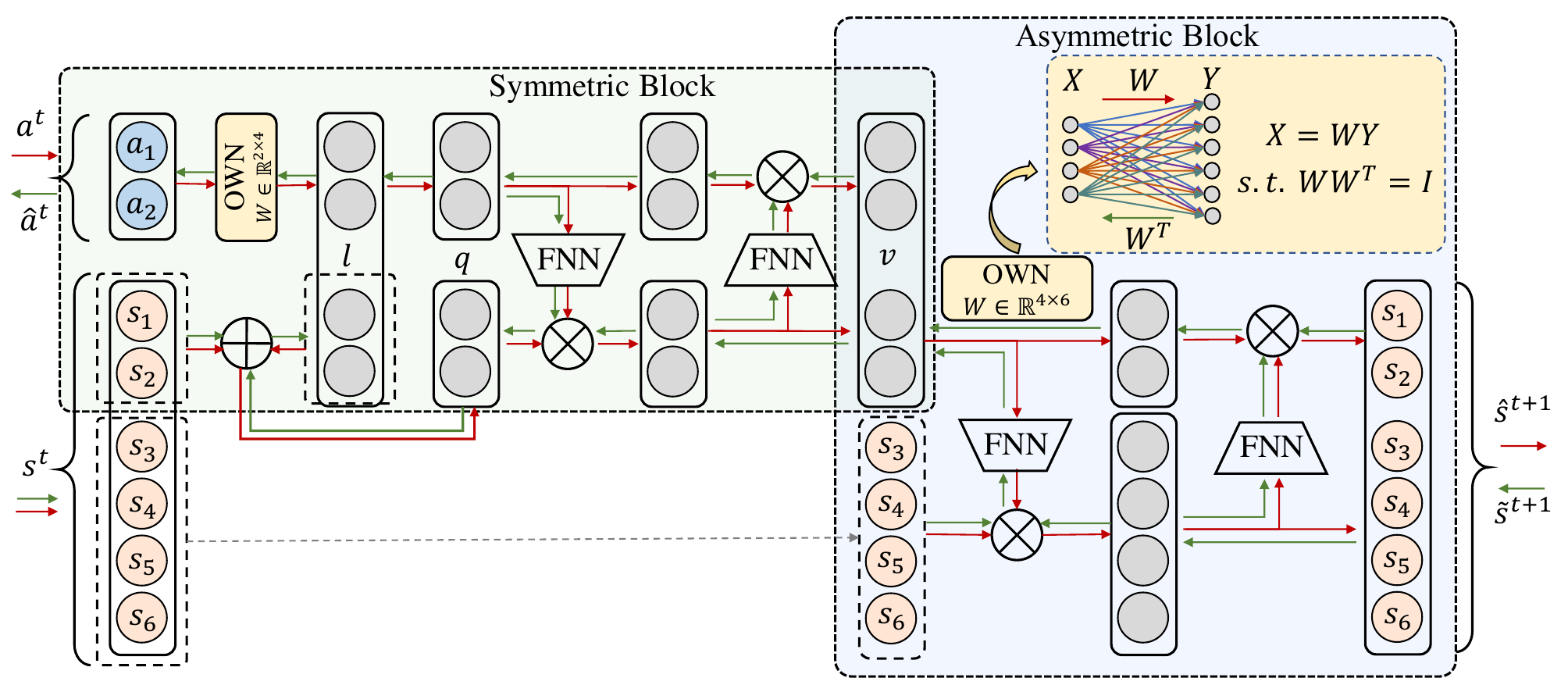}
        \centerline{(a) Architecture of CINN}
    \end{minipage}%
    \begin{minipage}[t]{0.3\linewidth}
        \centering
        \includegraphics[width=\textwidth]{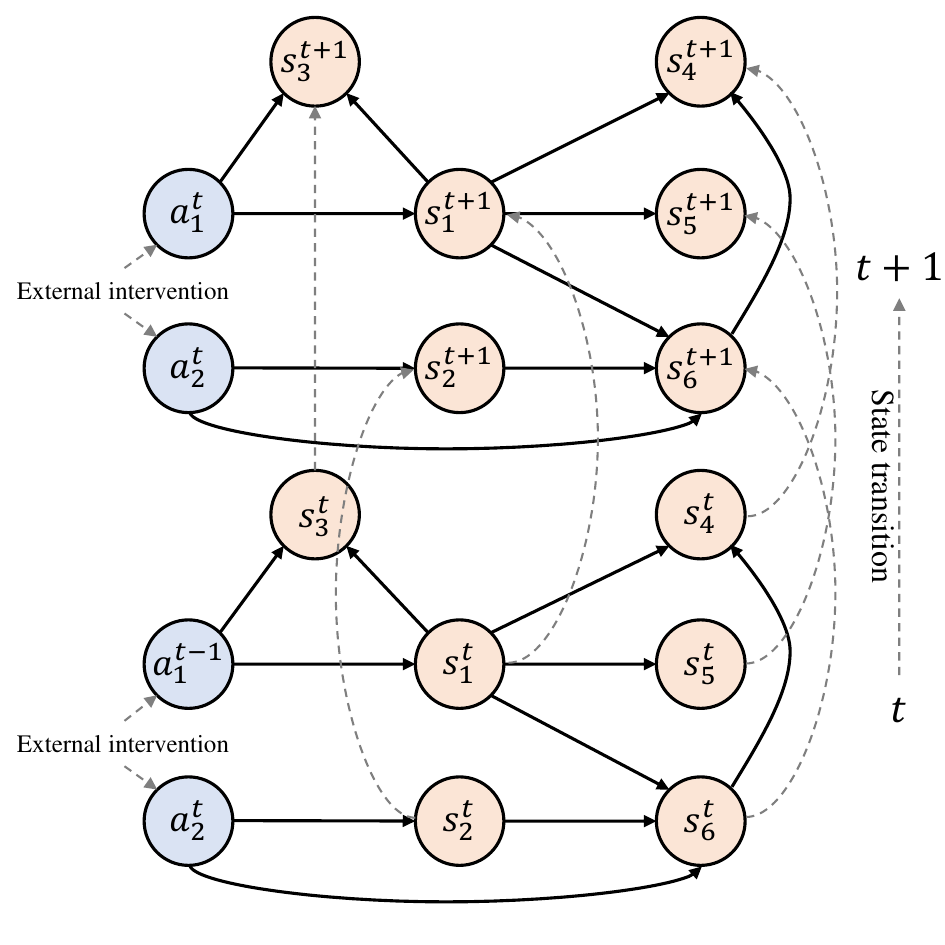}
        \centerline{(b) Causal Relationships}
    \end{minipage}
    \caption{(a) The architecture of Counterfactual Invertible Neural Networks (CINN) consists of two types of basic invertible blocks (symmetric and asymmetric). During forward prediction, the input consists of the intervention $a^t$ and the current state $s^t$, which are used to predict the future state $\hat{s}^{t+1}$. In counterfactual inference, given the current state $s^t$ and the desired future state $\tilde{s}^{t+1}$, CINN outputs the optimal action $\hat{a}^{t}$ necessary to achieve the target state.
(b) The diagram illustrates the causal relationships among the intervention $\{a^t_i\}_{i=1,2}$, the current state $\{s^t_i\}_{i=1,...,n}$ and the future state $\{s^{t+1}_i\}_{i=1,...,n}$ along with the temporal evolution of states under external intervention.}
    \label{fig2}
\end{figure*}

Given a causal graph $\mathcal{G}$, we propose a network structure generation algorithm based on the causal graph (detailed in Algorithm \ref{algorithm1}) to design the structure of the CINN. The core concept is to utilize the Floyd algorithm \citep{dibangoye2009topological} to determine the topological order of the directed acyclic causal graph. By organizing hierarchically according to this topological order, the predecessors and successors serve as inputs and outputs for the basic blocks. We use symmetric blocks when the dimensions of the predecessors and successors are the same. In cases where there is a mismatch, asymmetric blocks will handle the problem. For example, Figure \ref{fig2} (b) illustrates a causal graph with the external intervention $A^t\in\mathbb{R}^2$ and the states $S^t, S^{t+1}\in\mathbb{R}^6$, following the causal topological order: $a^t_1, a^t_2, s^{t+1}_1, s^{t+1}_2, s^{t+1}_3, s^{t+1}_6, s^{t+1}_5, s^{t+1}_4$. $s^{t+1}_1$ and $s^{t+1}_2$ are seperately influenced by $a^t_1$ and $a^t_2$. $s^{t+1}_3$ and $s^{t+1}_6$ are jointly affected by $a^t_1, a^t_2, s^t_1$, and $s^t_2$. Moreover, $s^{t+1}_5$ and $s^{t+1}_4$ are indirectly affected by $a^t_1$ and $a^t_2$. Therefore, the first symmetric block (corresponding to the first and second layers of the tree formed by the causal topological order.) represents the causal effects of $a^t_1$ and $a^t_2$ on $s^t_1$ and $s^t_2$, with input and output dimensions of $[a^t_1,a^t_2,s^t_1,s^t_2]\in\mathbb{R}^4$. For subsequent blocks, we consider reducing the number of blocks due to model complexity. Hence, we consider $s^t_3, s^t_4, s^t_5, s^t_6$ as additional inputs for the next layer, resulting in asymmetric blocks utilized to handle this situation. For real-world BG prediction for patients with T1D, where $S^t, S^{t+1}\in\mathbb{R}^{13}$, the dimension expansion results in the deepening of the tree corresponding to the causal topological graph, resulting in the use of the two symmetric blocks and one asymmetric block.

\subsection{Basic Blocks\label{section31}}

Since the specific block concatenation and construction of CINN rely primarily on the causal relationships that support its interpretability, we will introduce the internal implementation details of each type of basic block in the following.

As depicted in Figure \ref{fig2} (a), during the forward prediction process, CINN takes the state $s^t$ at time $t$ and the external intervention $a^t$ as inputs and outputs the predicted future state $\hat{s}^{t+1}$. This can be formalized as $\mathbb{E}[S^{t+1}=\hat{s}^{t+1}|do(A^t=a^t),S^t=s^t)$. In the counterfactual inference process, CINN takes the state $s^t$ at time $t$ and the desired state $\tilde{s}^{t+1}$ at the next time $t+1$ as inputs and outputs the optimal action $\hat{a}^{t}$ to achieve the desired state. This can be formalized as $\mathbb{E}[A^t=\hat{a}^{t}|do(S^{t+1}=\tilde{s}^{t+1}),S^t = s^t)$. As Equation \ref{equationapp_3} indicates, $m$ and $n$ can be any neural network structure. For lightweight models, in CINN, the design of these two types of basic blocks consists of fully connected feedforward neural networks (FNN) equipped with leaky ReLU activations \citep{xu2020reluplex}.

\subsubsection{Symmetric Blocks}

In symmetric blocks, where the inputs and outputs are of the same dimension, i.e., $u, v \in \mathbb{R}^d$, we first partition $u$ into $u_1, u_2$ based on the actual involvement of these $d$-dimensions in the computations. Next, we will  achieve the fusion of information between $u_1$ and $u_2$ with orthogonal weight matrices:

\begin{equation}
	\begin{aligned}
	[l_1,l_2]=u_1 W;\quad q_1 =l_1;\quad q_2=l_2+u_2
	\end{aligned}
    \label{equation3_1}
\end{equation}

\begin{algorithm}[!t]
    \caption{Structure Generation for CINN}
    \label{algorithm1}
    \textbf{Input}: Causal graph $\mathcal{G}$\\
    \textbf{Output}: The structure of CINN
    \begin{algorithmic}[1] %[1] enables line numbers
    \STATE // Calculation of Topological Order of Causal Graph Using Floyed Algorithm
    \STATE $\rm T=\rm{Floyed}(\mathcal{G}$)
    \STATE $\rm{predecessor}=0, \rm{successor}=0, \rm{blocks}=\rm{list}()$
    \STATE // Cyclically generating all basic blocks, including symmetric and asymmetric types.\
    \FOR {$\rm{i}$ \TO $\rm{LayerNum}(T)-1$} 
        \STATE // Calculating the predecessor and successor dimensions for each basic block.
        \STATE $\rm{predecessor}+= T[i]$
        \STATE $\rm{successor}=T[i+1]$
        \STATE // Generating symmetric or asymmetric blocks.
        \IF{predecessor==successor}
        \STATE Create a symmetric block and add it to blocks
        \ENDIF
        \IF{predecessor!=successor}
        \STATE Create a asymmetric block and add it to blocks
        \ENDIF
    \ENDFOR
    \STATE \textbf{return} blocks
    \end{algorithmic}
\end{algorithm}

Here, we use orthogonal weight matrices to achieve linear information fusion within $u$, where $WW^T = \mathbb{1}$. Subsequently, we implement nonlinear functional relations based on Equation \ref{equationapp_3}. Following the causal topological relations, we can stack the invertible blocks multiple times to simulate arbitrarily complex function transformations within CINN. Next, $q_1$ and $q_2$ undergo standard affine coupling layers, and after two types of transformations, the output $v$ is obtained to characterize the resulting hidden representation, where $m_1$, $m_2$, $n_1$, and $n_2$ are implemented as FNN with leaky ReLU:

\begin{equation}
    \begin{aligned}
    \hat{v}_1 &= q_1\odot\exp(m_2(q_2))+n_2(q_2)\\
    \hat{v}_2 &= q_2\odot\exp(m_1(\hat{v}_1)) + n_1(\hat{v}_1)\\
    \end{aligned}
    \label{equation3_2}
\end{equation}

The counterfactual inference process within the symmetric block processes the predefined output hidden representation $\tilde{v}$ along with the partially computable input hidden representation $\tilde{u}_2$ to yield the desired $\hat{u}_1$. In the forward prediction process, transitioning from $q$ to $v$, the reversible nature of neural networks allows us to derive $\hat{q}$ from $\tilde{v}$ without $\tilde{u}_2$.

\begin{equation}
    \begin{aligned}
    \hat{q}_2 &= (\tilde{v}_2 - n_1(\tilde{v}_1)) \odot \exp(-m_1(\tilde{v}_1))\\
    \hat{q}_1 &= (\tilde{v}_1 - n_2(\hat{q}_2)) \odot \exp(-m_2(\hat{q}_2))\\
    \end{aligned}
    \label{equation3_3}
\end{equation}

Subsequently, the process of inferring $\hat{u}$ from $\hat{q}$ in the counterfactuals requires the aid of the partially computable input hidden representation $\tilde{u}_2$. Specifically, for $\hat{l}_1=\hat{q}_1$ and $\hat{l}_2=\hat{q}_2-\tilde{u}_2$, according to $WW^T=I$, we have $\hat{u}_1=\hat{l}W^T$.

Let us consider the symmetric block in Figure \ref{fig2} (a) as an illustration. The input in this symmetric block is denoted as $u=[a^t_1,a^t_2,s^t_1,s^t_2]\in\mathbb{R}^4$. Naturally, we first decompose $u$ into $u_1=[a^t_1,a^t_2]$ and $u_2=[s^t_1,s^t_2]$, which physically correspond to the external intervention $a^t$ and the state $s^t$, respectively. To simulate that $a^t$ is a consequence of $s^t$, we apply an orthogonal weight matrix on $u_1$ and then add $u_2$ to it, resulting in $q$. Subsequently, standard affine coupling layers are used to subject $q$ to complicated nonlinear transformations, leading to $\hat{v}$, while ensuring the reversibility of the transformation. Hence, counterfactual inference also exists in the inverse transformation form, as described above.

In the symmetric blocks, forward prediction can be formally represented as $\mathbb{E}[\hat{v}|do(A^t = a^t), S^t = s^t]$ and counterfactual inference as $\mathbb{E}[A^t=\hat{a}^t|\tilde{v}, S^t = s^t]$. However, it is important to note that the symmetric blocks can only accommodate scenarios where the input dimension matches the output dimension. As a result, neural networks constructed using only symmetric blocks have certain limitations in practical scenarios.

\subsubsection{Asymmetric Blocks}

Specifically, for forward prediction, $|A^t|+|S^t|\neq|S^{t+1}|$, and for counterfactual inference, $|S^t|+|S^{t+1}|\neq|A^t|$. As shown in the asymmetric block in Figure \ref{fig2} (a), the output dimension of the symmetric block during forward prediction is $\hat{v}\in\mathbb{R}^4$. Additionally, based on the causal topological relations, $\hat{v}$ is concatenated with the remaining state $s^t_{3:6}$ to finally obtain the predicted state $\hat{s}^{t+1}$. However, $[\hat{v};s_3,...,s_6]\in\mathbb{R}^8$ while $\hat{s}^{t+1}\in\mathbb{R}^6$, which poses a dimensionality mismatch problem.

To address the dimensionality transformation problem, we have made certain modifications to the symmetric block and proposed an asymmetric type of the basic block that incorporates the ability to adjust the input dimension using orthogonal weight matrices. This asymmetric type of basic block allows CINN to effectively handle scenarios where the input dimension does not match the output dimension. Noting that in the forward prediction process, $a^t$ has already transmitted information to $s_1$ and $s_2$, we do not use orthogonal weight matrices for information fusion in the asymmetric block. Instead, we use its dimension transformation capability to transform the dimensions to match those of the final outputs. This way we can end up with $\hat{s}^{t+1}\in\mathbb{R}^6$. The counterfactual inference will also hold because of the invertibility of CINN. Combining these two types of blocks in series will effectively increase the usability and versatility of the invertible model.

\subsection{Training\label{section33}}

Overall, during the forward prediction, the external intervention $a^t$ and the state $s^t$ are inputted into CINN. The hidden representations of the causal relationships are used to predict the future state $s^{t+1}$. Conversely, using the same network and parameters, CINN is fed with the state $s^t$ and the desired future state $\tilde{s}^{t+1}$ to perform the counterfactual process and derive the external intervention $\hat{a}^t$. Consequently, the loss function comprises two components: the forward prediction loss $\mathcal{L}_F(s^{t+1},\hat{s}^{t+1})=\frac{1}{n}\sum_{i=1}^{n}(s^{t+1}_{(i)}-\hat{s}^{t+1}_{(i)})^2$ and the counterfactual inference loss $\mathcal{L}_I(a^{t},\hat{a}^{t})=\frac{1}{2}\sum_{i=1}^{2}(a^{t}_{(i)}-\hat{a}^{t}_{(i)})^2$. We define the loss function as the weighted sum of these bidirectional losses:

\begin{equation}
	\begin{aligned}
    \mathcal{L}& =\mathcal{L}_F+\mathcal{L}_I
	\end{aligned}
    \label{equation5}
\end{equation}
The loss is computed as the average over the shuffled mini-batch, and the derivatives of each parameter can be computed via alternating bidirectional back-propagation.

\begin{figure}[t]
    \centering
    \includegraphics[width=0.7\linewidth]{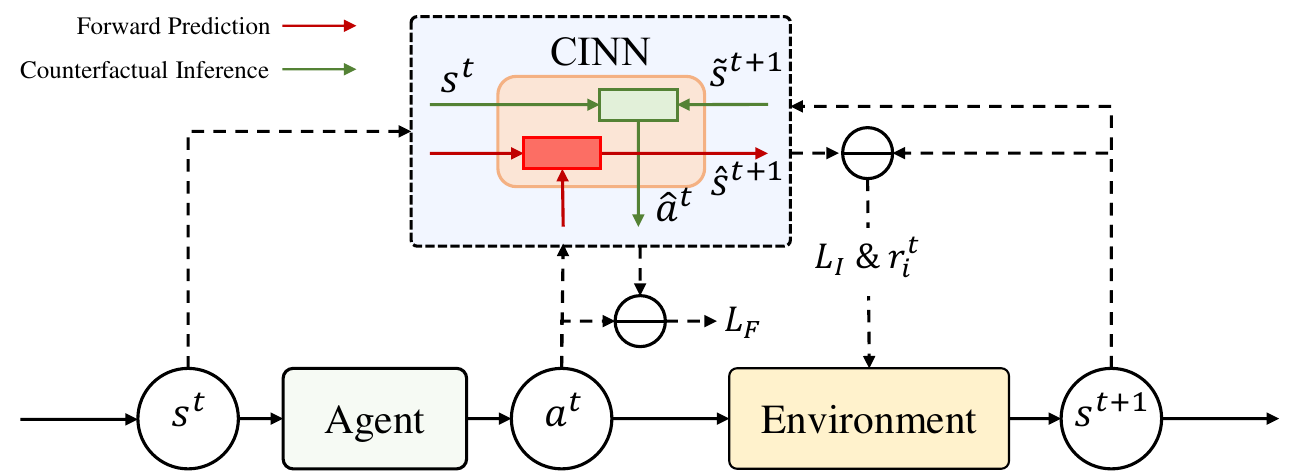} % Reduce the figure size so that it is slightly narrower than the column. Don't use precise values for figure width.This setup will avoid overfull boxes.
    \caption{RL with pre-trained frozen introspective CINN.}
    \label{fig3}
\end{figure}

\section{Introspective Reinforcement Learning}

RL algorithms strive to learn continuous decision-making policies that accomplish specific tasks by maximizing the cumulative rewards provided by the environment \citep{schulman2017proximal,haarnoja2018soft}. In real-life situations, we often encounter events that may occur only once or twice in a lifetime. Therefore, agents need the ability to introspect and imagine based on previous trajectories before taking action rather than randomly exploring. 
Since we have obtained a pre-trained CINN based on Section \ref{CINN}, in this Section, we consider the pre-trained CINN as a mechanism for introspection and aim to integrate it into RL in a frozen block to enhance the safety and stability of BG control.

\begin{figure}[t]
    \begin{minipage}[t]{0.5\linewidth}
        \centering 
        \includegraphics[width=\textwidth]{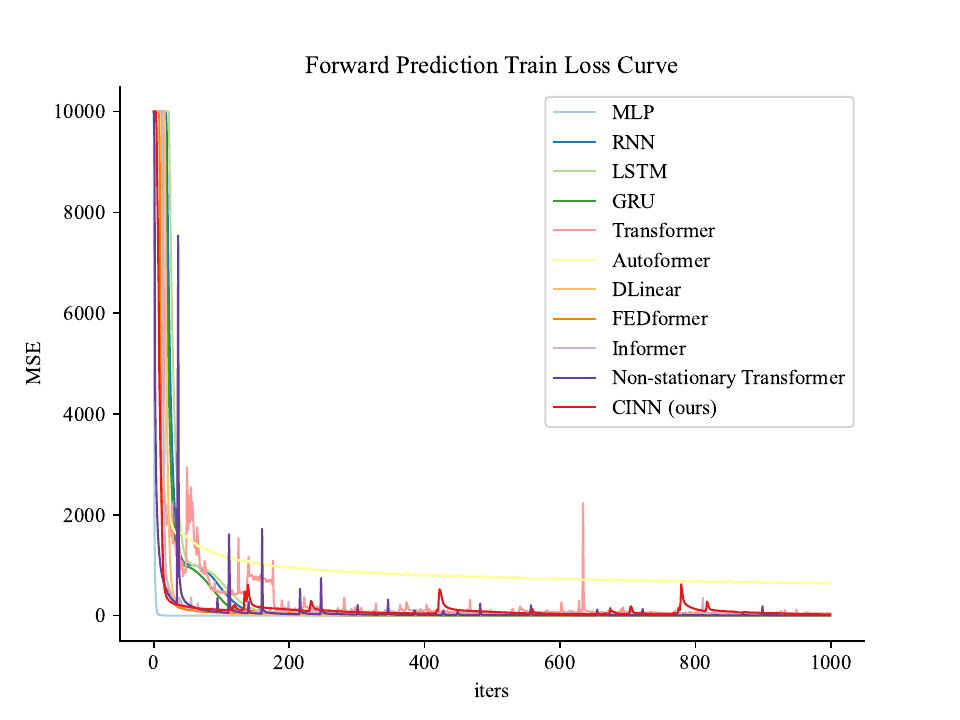}
        % \centerline{(a) Forward Prediction Train Loss}
    \end{minipage}%
    \begin{minipage}[t]{0.5\linewidth}
        \centering
        \includegraphics[width=\textwidth]{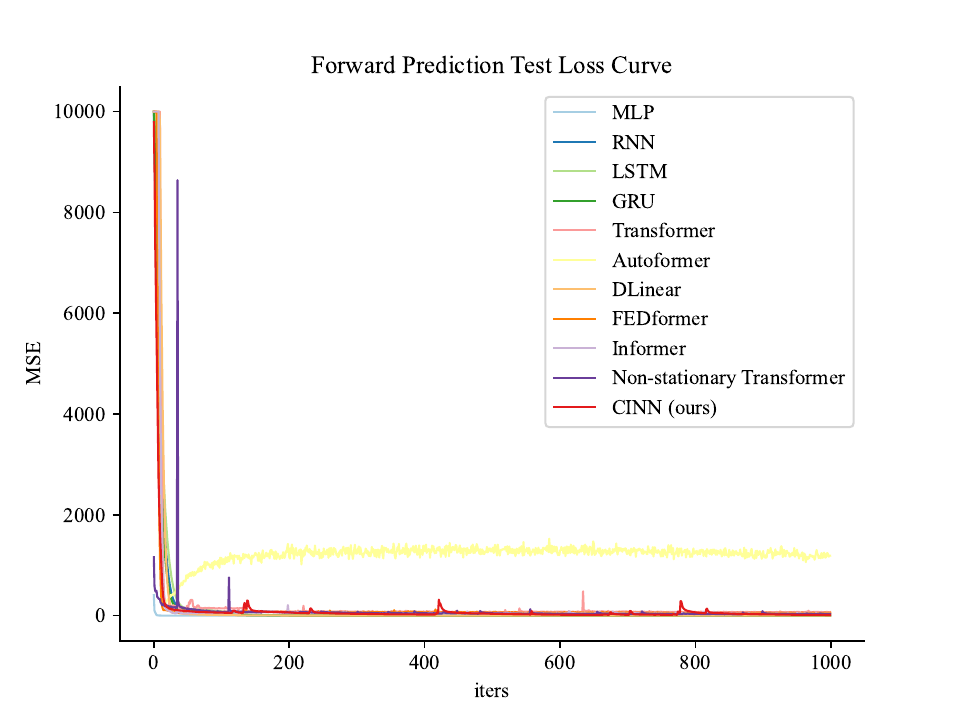}
        % \centerline{(b) Forward Prediction Test Loss}
    \end{minipage}

    \begin{minipage}[t]{0.5\linewidth}
        \centering 
        \includegraphics[width=\textwidth]{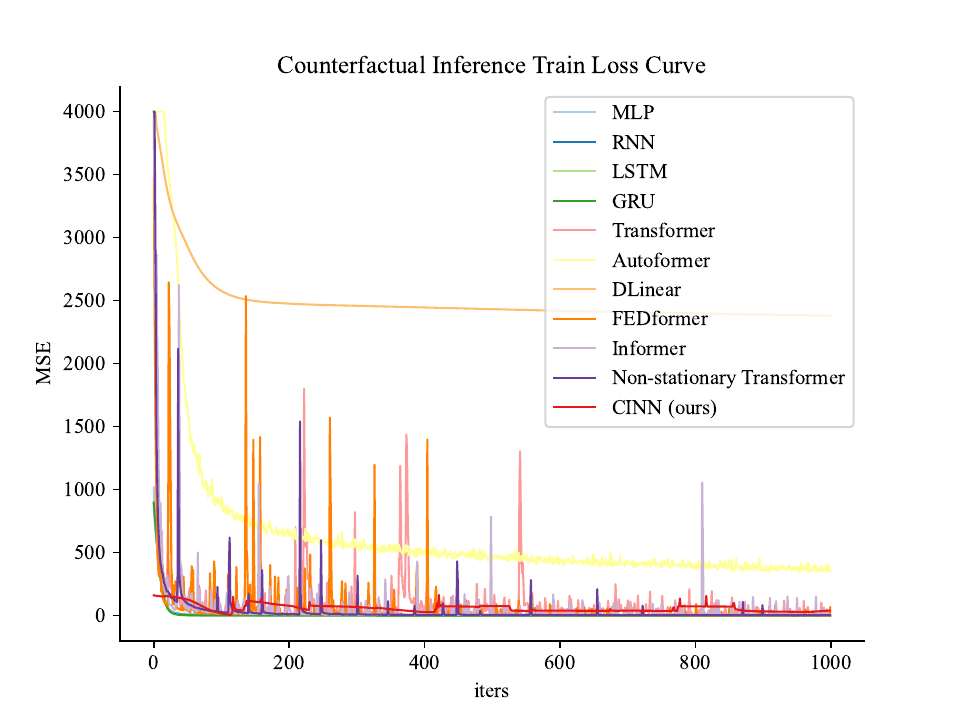}
        % \centerline{(c) Counterfactual Inference Train Loss}
    \end{minipage}%
    \begin{minipage}[t]{0.5\linewidth}
        \centering
        \includegraphics[width=\textwidth]{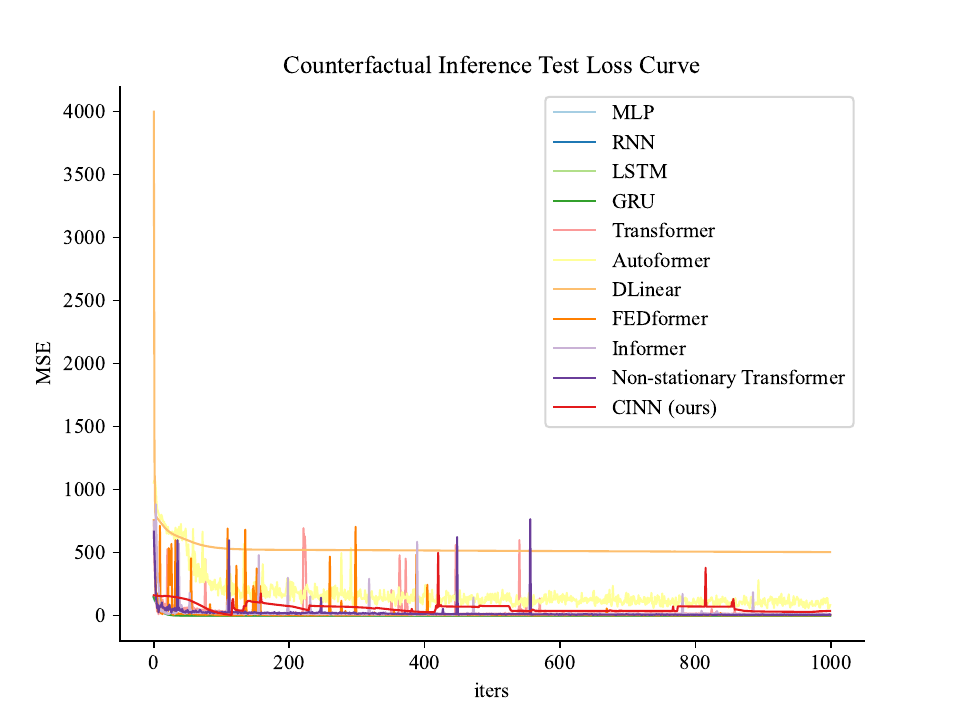}
        % \centerline{(d) Counterfactual Inference Test Loss}
    \end{minipage}
    \caption{Forward prediction and counterfactual inference under IID scenarios.}
    \label{fig4}
\end{figure}

\begin{figure*}[t]
    \begin{minipage}[t]{0.24\linewidth}
        \centering 
        \includegraphics[width=\textwidth]{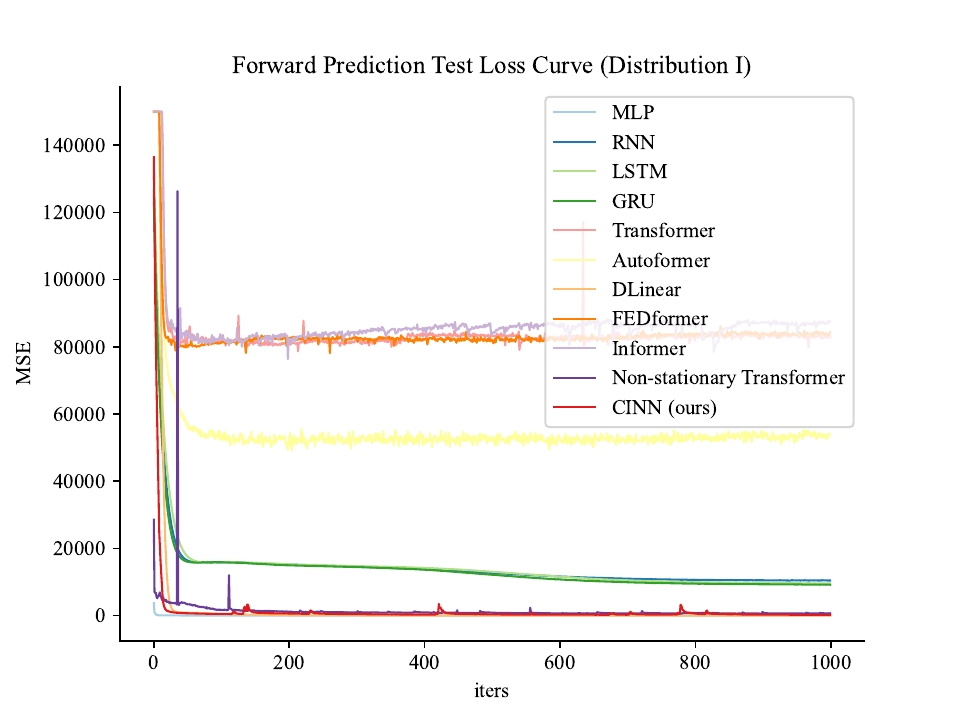}
        % \centerline{Distribution I}
    \end{minipage}%
    \begin{minipage}[t]{0.24\linewidth}
        \centering
        \includegraphics[width=\textwidth]{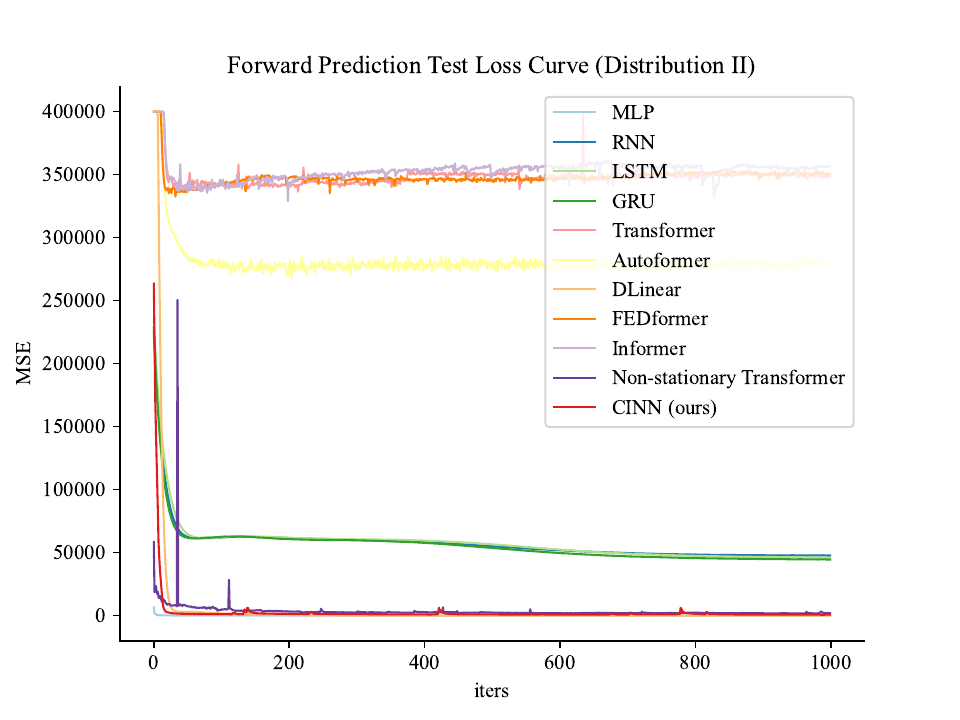}
        % \centerline{Distribution II}
    \end{minipage}
    \begin{minipage}[t]{0.24\linewidth}
        \centering 
        \includegraphics[width=\textwidth]{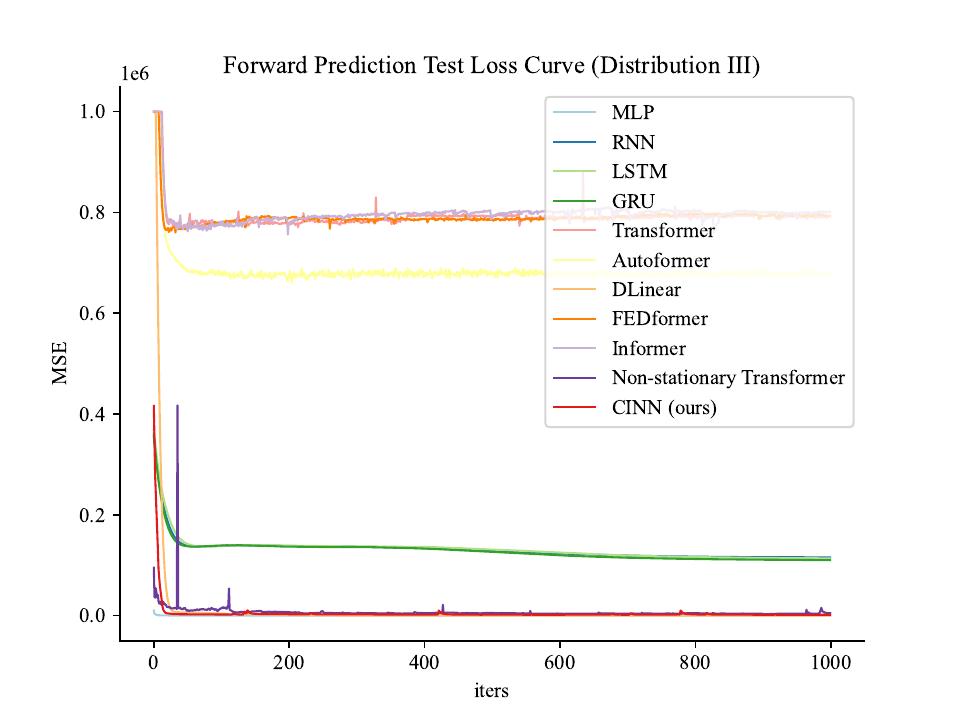}
        % \centerline{Distribution III}
    \end{minipage}%
    \begin{minipage}[t]{0.24\linewidth}
        \centering
        \includegraphics[width=\textwidth]{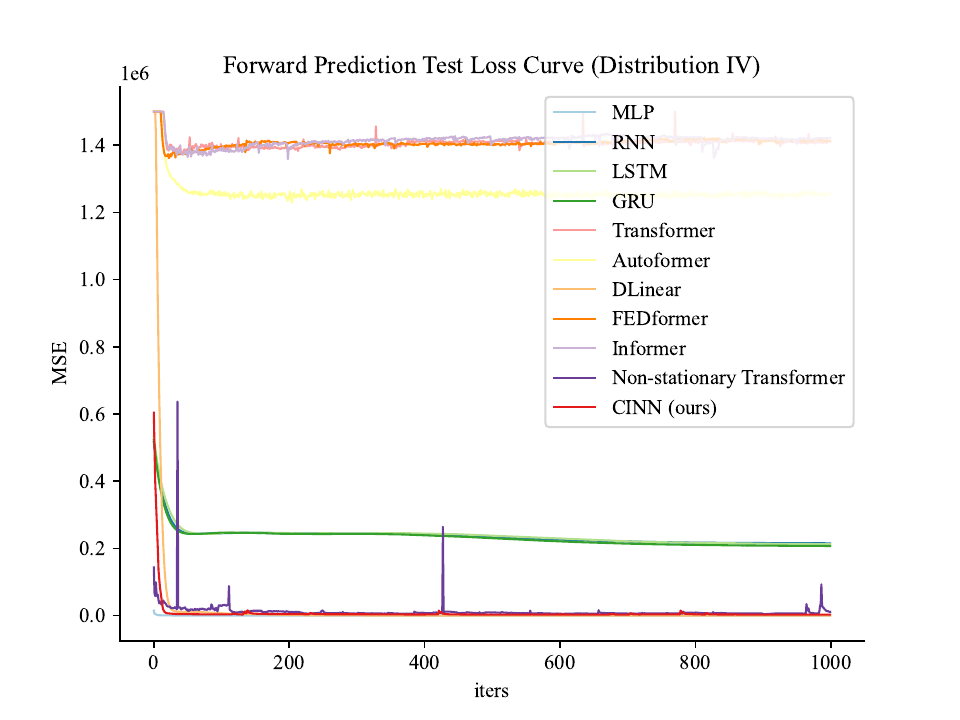}
        % \centerline{Distribution IV}
    \end{minipage}

    \begin{minipage}[t]{0.24\linewidth}
        \centering 
        \includegraphics[width=\textwidth]{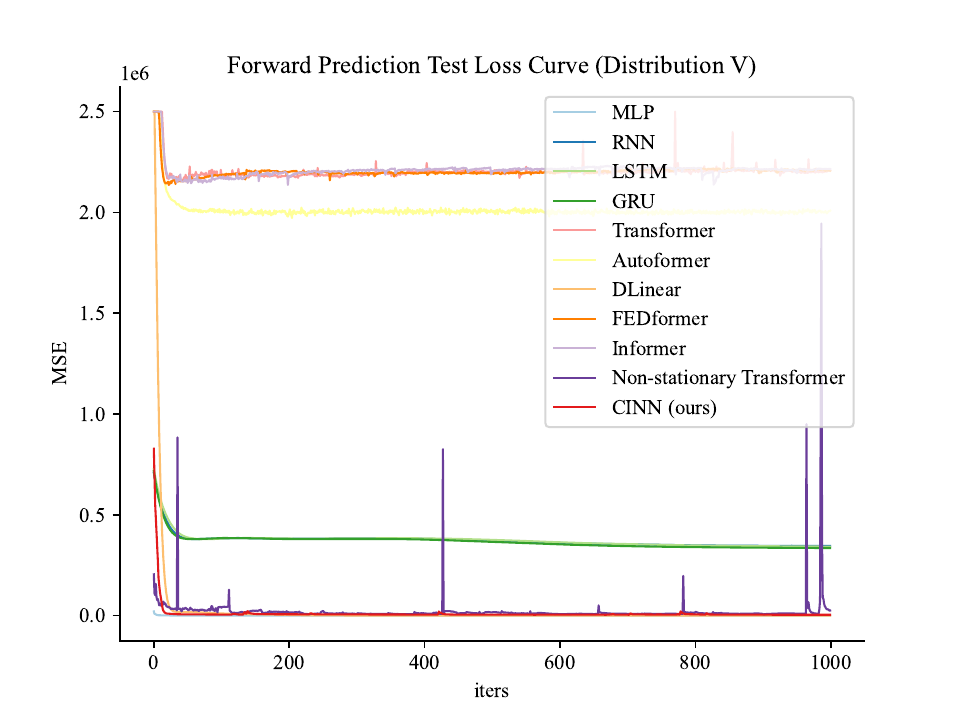}
        % \centerline{Distribution V}
    \end{minipage}%
    \begin{minipage}[t]{0.24\linewidth}
        \centering
        \includegraphics[width=\textwidth]{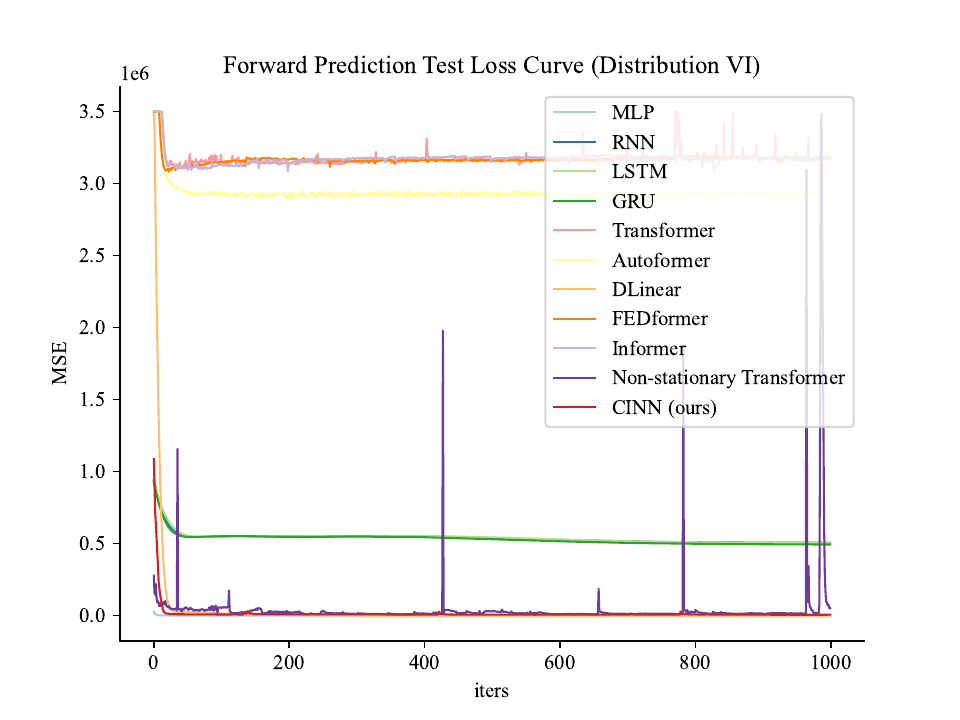}
        % \centerline{Distribution VI}
    \end{minipage}
    \begin{minipage}[t]{0.24\linewidth}
        \centering 
        \includegraphics[width=\textwidth]{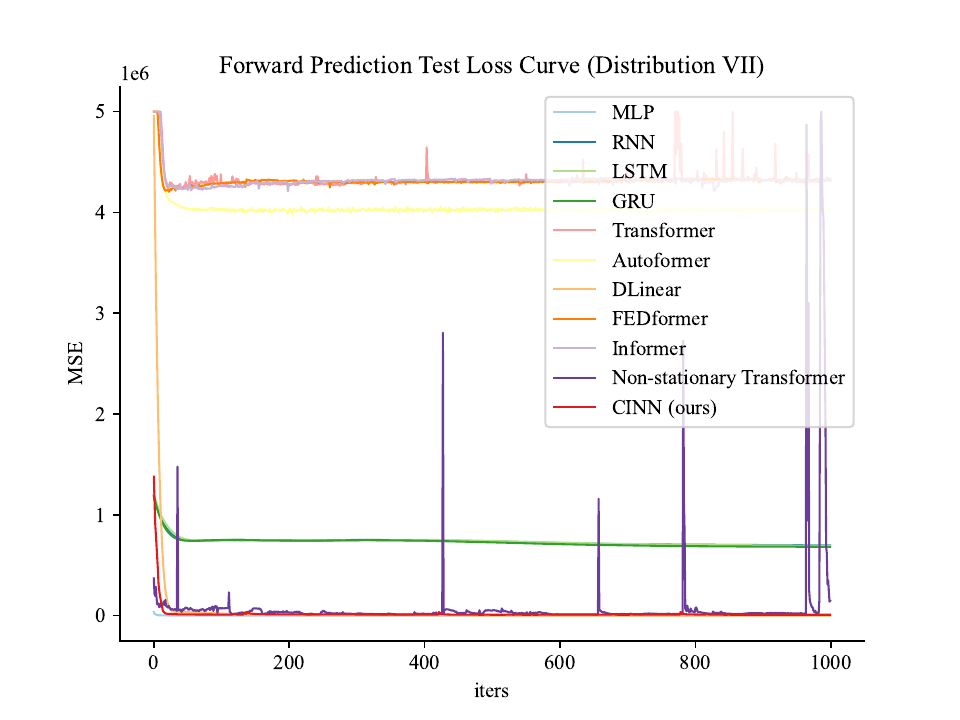}
        % \centerline{Distribution VII}
    \end{minipage}%
    \begin{minipage}[t]{0.24\linewidth}
        \centering
        \includegraphics[width=\textwidth]{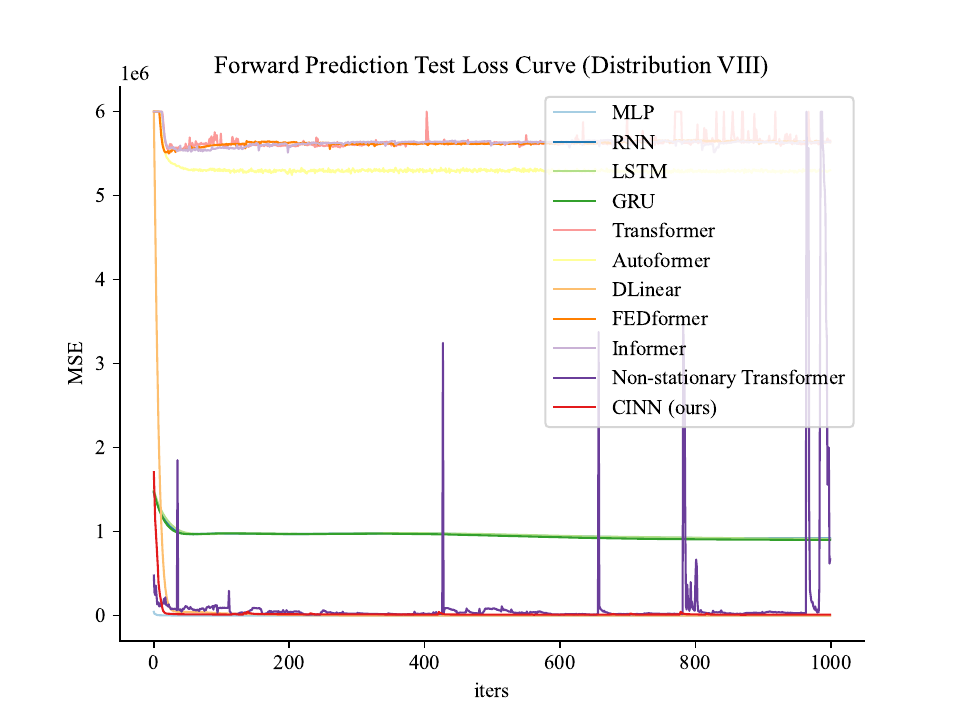}
        % \centerline{Distribution VIII}
    \end{minipage}

    \begin{minipage}[t]{0.24\linewidth}
        \centering 
        \includegraphics[width=\textwidth]{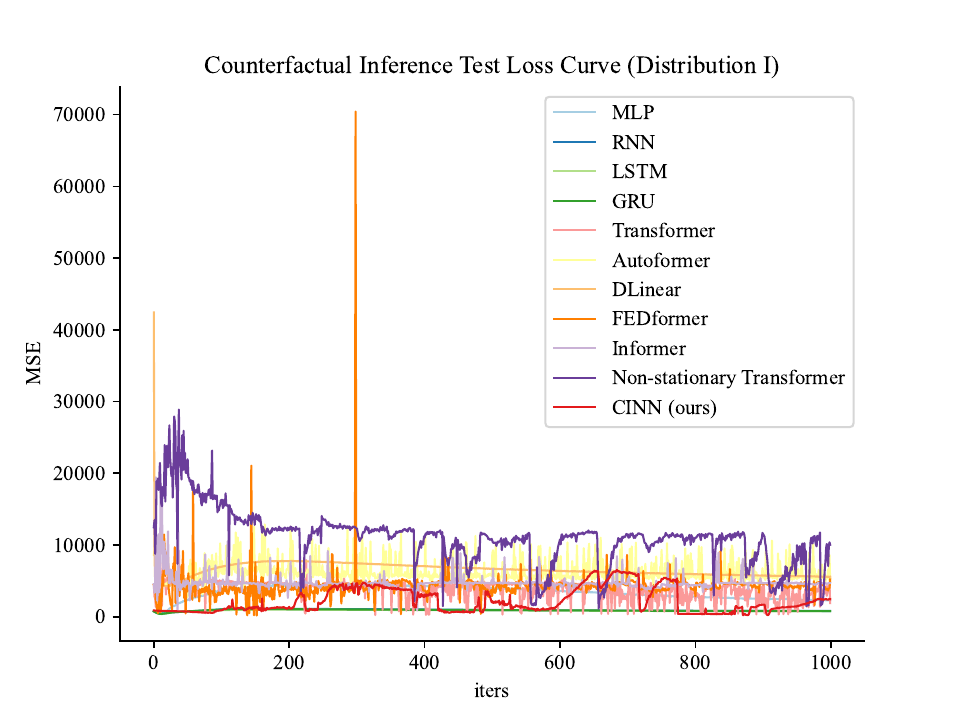}
        % \centerline{Distribution I}
    \end{minipage}%
    \begin{minipage}[t]{0.24\linewidth}
        \centering
        \includegraphics[width=\textwidth]{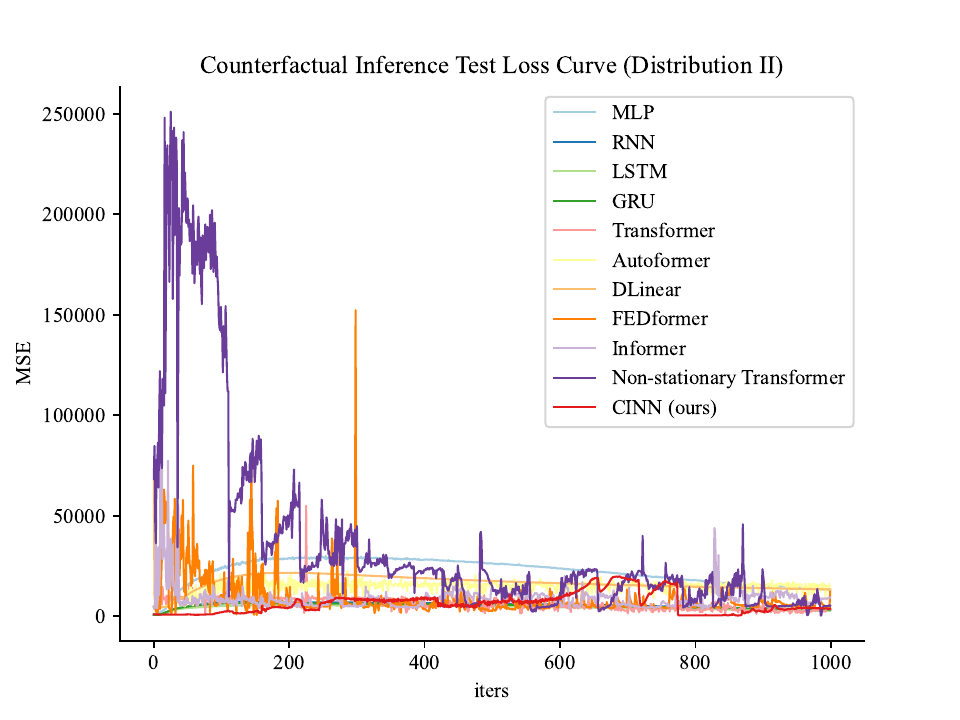}
        % \centerline{Distribution II}
    \end{minipage}
    \begin{minipage}[t]{0.24\linewidth}
        \centering 
        \includegraphics[width=\textwidth]{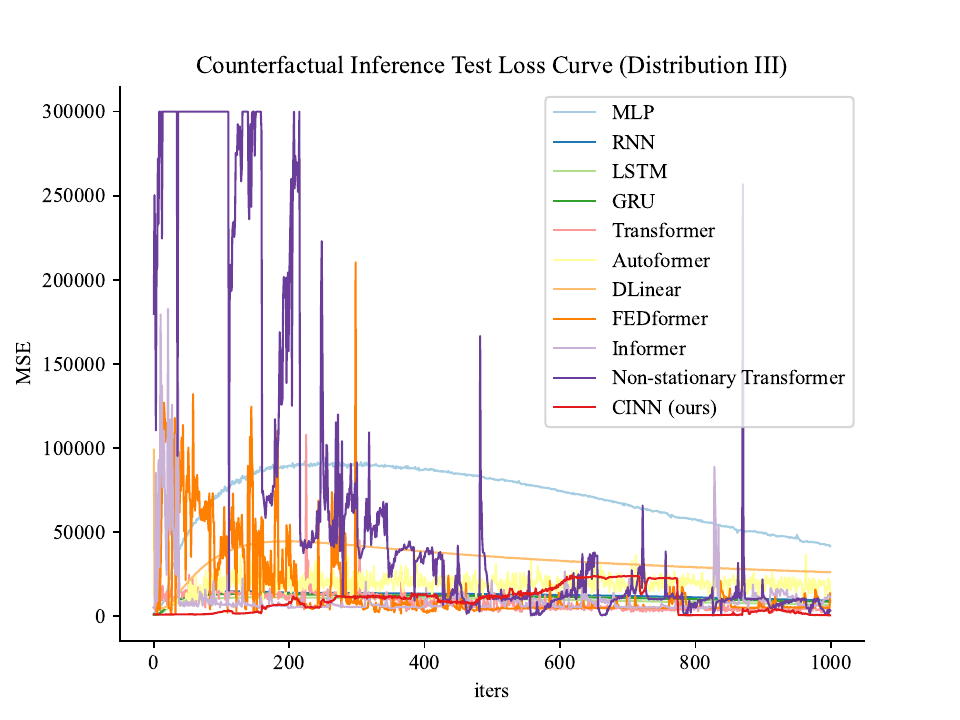}
        % \centerline{Distribution III}
    \end{minipage}%
    \begin{minipage}[t]{0.24\linewidth}
        \centering
        \includegraphics[width=\textwidth]{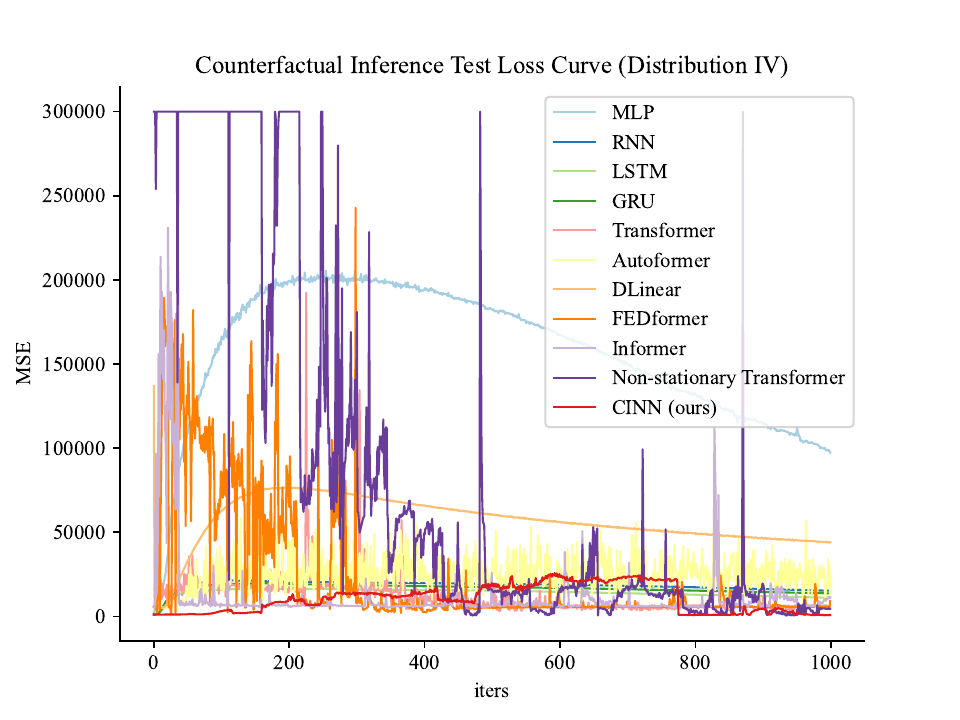}
        % \centerline{Distribution IV}
    \end{minipage}

    \begin{minipage}[t]{0.24\linewidth}
        \centering 
        \includegraphics[width=\textwidth]{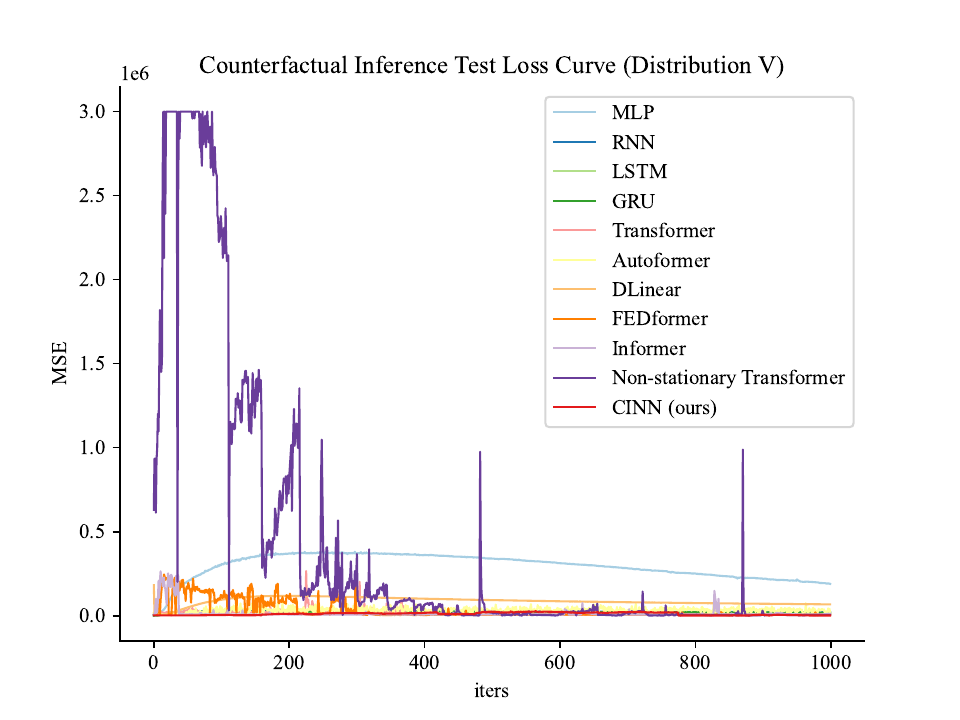}
        % \centerline{Distribution V}
    \end{minipage}%
    \begin{minipage}[t]{0.24\linewidth}
        \centering
        \includegraphics[width=\textwidth]{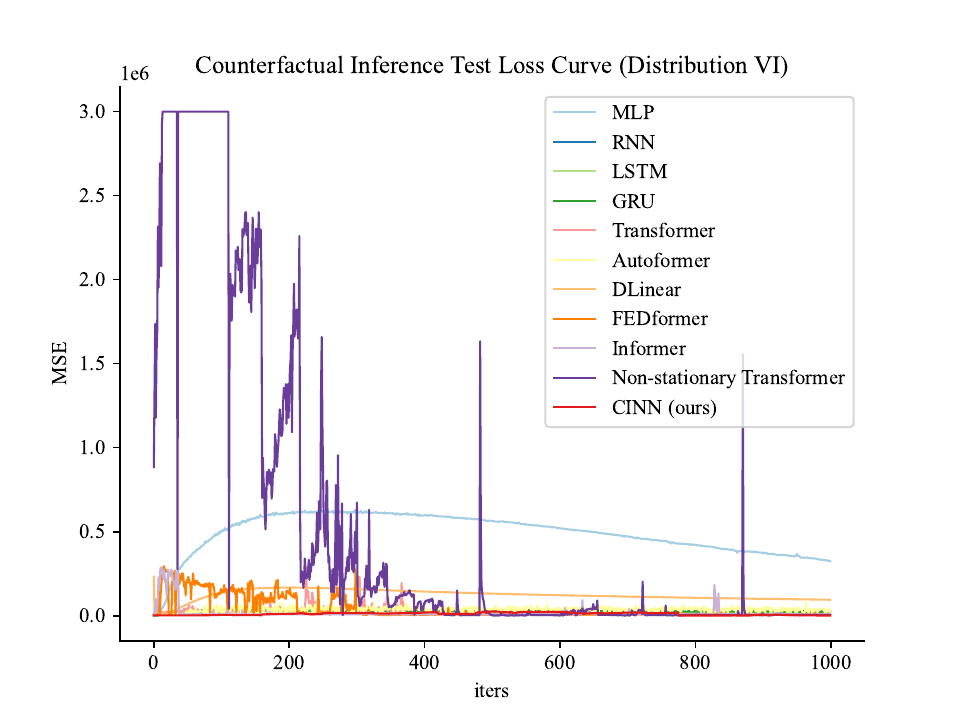}
        % \centerline{Distribution VI}
    \end{minipage}
    \begin{minipage}[t]{0.24\linewidth}
        \centering 
        \includegraphics[width=\textwidth]{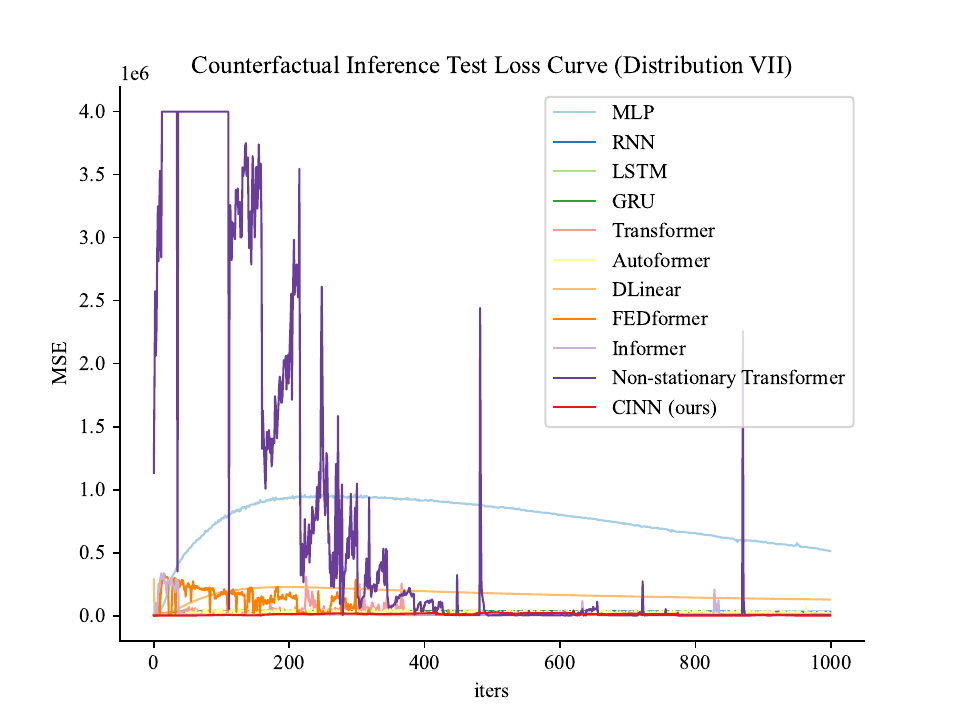}
        % \centerline{Distribution VII}
    \end{minipage}%
    \begin{minipage}[t]{0.24\linewidth}
        \centering
        \includegraphics[width=\textwidth]{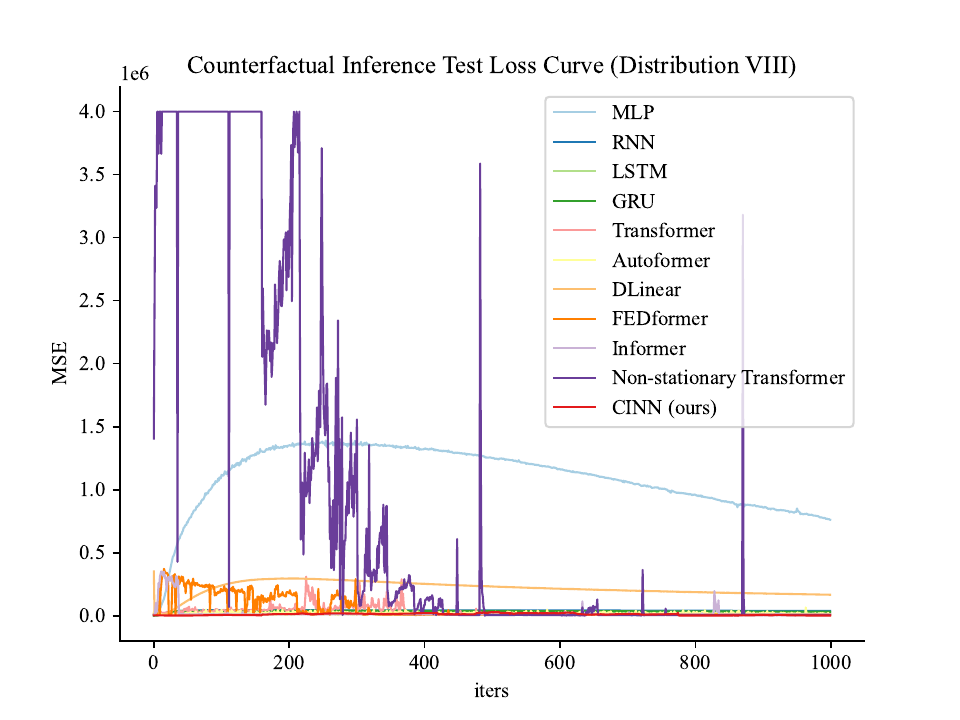}
        % \centerline{Distribution VIII}
    \end{minipage}

    \caption{Forward prediction and counterfactual inference of CINN and baselines under OOD scenarios.}
    \label{fig5}
\end{figure*}

\subsection{Prediction-driven Exploration}

We denote the RL policy as $\pi(s^t;\psi)$, represented by a deep neural network with parameters $\psi$. Given the state $s^t$ of patients with T1D, the agent executes the action $a^t\sim\pi(s^t;\psi)$ sampled from the policy. The parameter $\psi_P$ is optimized to maximize the expected sum of rewards, $\psi_P=\max_{\psi}\mathbb{E}_{\pi(s^t;\psi)}[\sum^t r^t]$. Similar to \citep{yu2023arlpe}, the reward $r^t$ means the degree of deviation of the patient's BG level $s^{t}_{G}$ from the target range $r=[70,180]$, where $r_{\min}=70$ and $r_{\max}=180$:

\begin{equation}
	\begin{aligned}
    r^t=&(\max(\min(s^{t}_{G}-r_{\min}, s^{t}_{G}-r_{\max}), \min(r_{\min}-s^{t}_{G}, r_{\max}-s^{t}_{G}), 0))^2
	\end{aligned}
    \label{equation4_1}
\end{equation}

Moreover, the environment is a simulator for patients with T1D \citep{man2014uva}. During the continuous interaction between the agent and the environment, the agent improves its understanding of the underlying dynamics of the environment. Once the agent comprehends the transition probabilities of states about patients with T1D, denoted as $p(s^{t+1}|s^t,a^t)$, it can accurately predict the future state, facilitating the generation of the optimal BG control policy. Conversely, discrepancies between the actual future state $s^{t+1}$ provided by the environment and the predicted state $\hat{s}^{t+1}$ can make the agent think about the inaccuracy in its predictions. The inaccuracy often manifests itself as a reward $r^t_i$, which motivates the agent to explore as curiosity \citep{pathakICMl17curiosity}. It is worth noting that this reward, as an intrinsic reward, is generated internally within the agent and cannot be directly observed. 
Different from curiosity, to take safer actions, we formalize the intrinsic reward as follows:

\begin{figure*}[t]
    \begin{minipage}[t]{0.33\linewidth}
        \centering 
        \includegraphics[width=\textwidth]{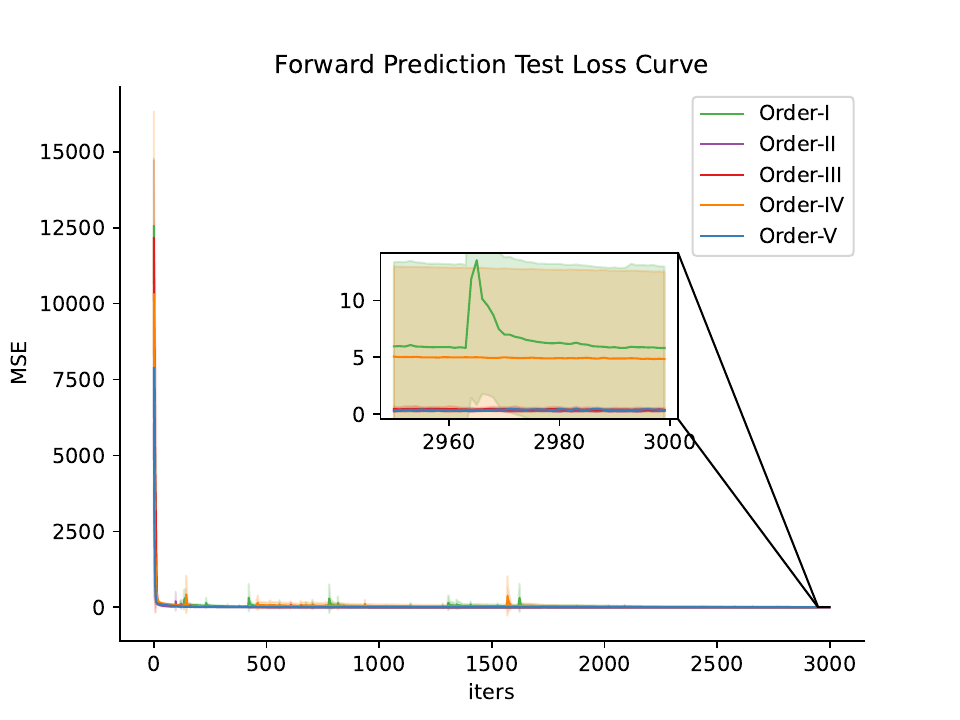}
        % \centerline{Distribution I}
    \end{minipage}%
    \begin{minipage}[t]{0.33\linewidth}
        \centering
        \includegraphics[width=\textwidth]{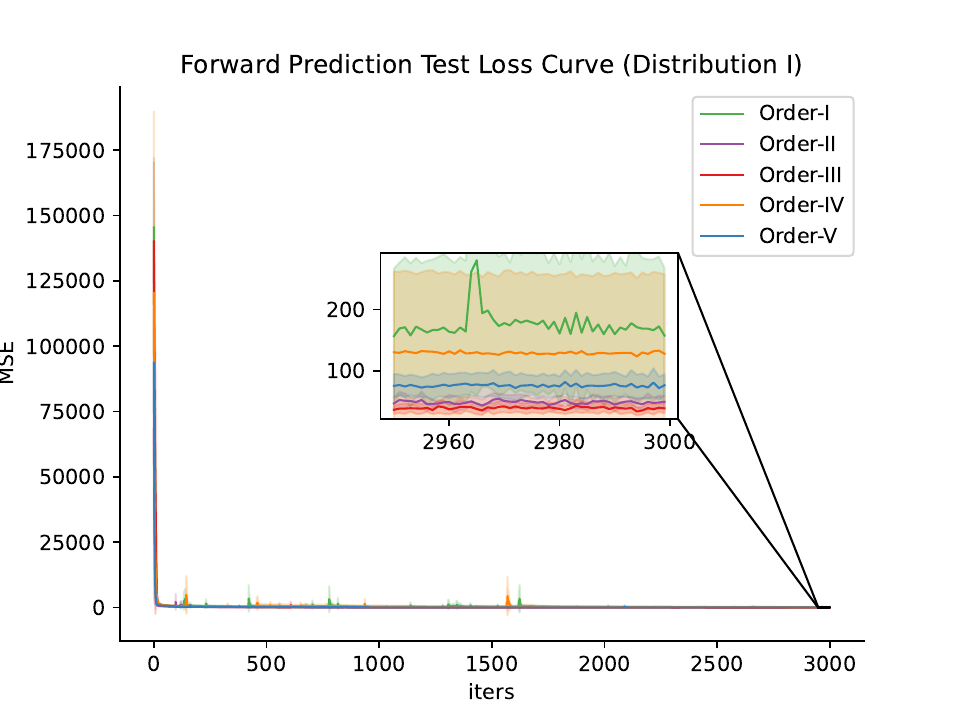}
        % \centerline{Distribution II}
    \end{minipage}
    \begin{minipage}[t]{0.33\linewidth}
        \centering 
        \includegraphics[width=\textwidth]{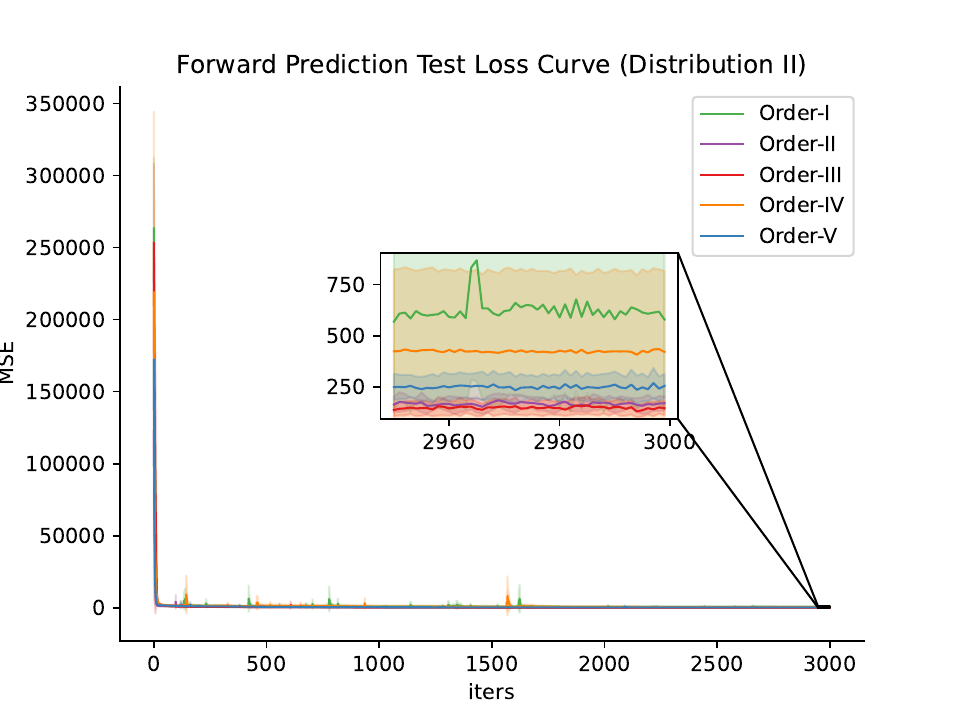}
        % \centerline{Distribution III}
    \end{minipage}%

    \begin{minipage}[t]{0.33\linewidth}
        \centering 
        \includegraphics[width=\textwidth]{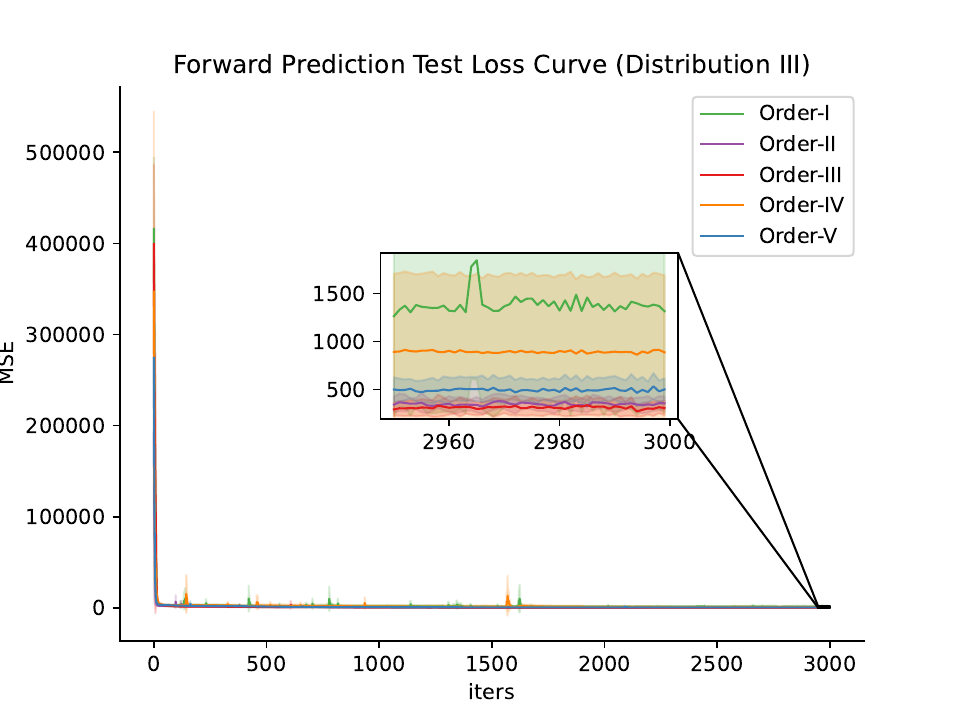}
        % \centerline{Distribution III}
    \end{minipage}%
    \begin{minipage}[t]{0.33\linewidth}
        \centering
        \includegraphics[width=\textwidth]{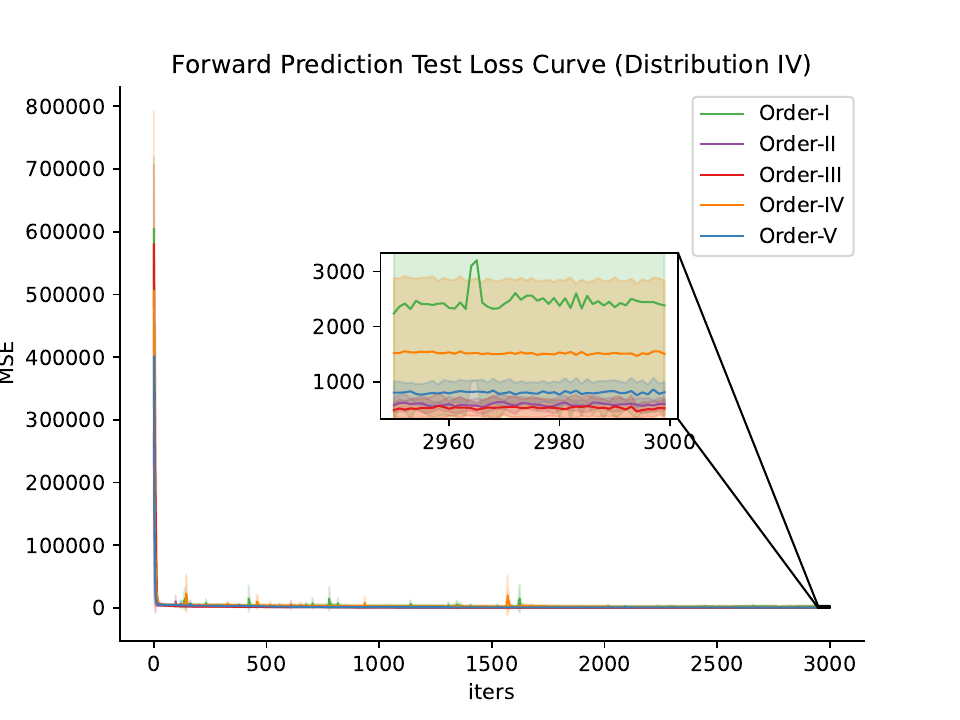}
        % \centerline{Distribution IV}
    \end{minipage}
    \begin{minipage}[t]{0.33\linewidth}
        \centering 
        \includegraphics[width=\textwidth]{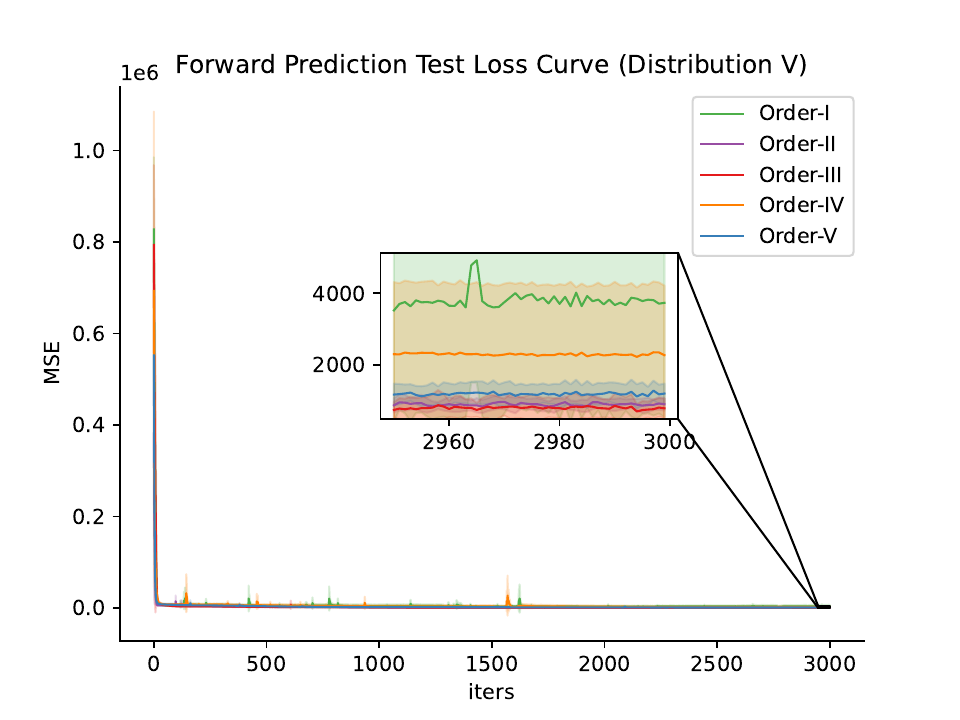}
        % \centerline{Distribution V}
    \end{minipage}%

    \begin{minipage}[t]{0.33\linewidth}
        \centering
        \includegraphics[width=\textwidth]{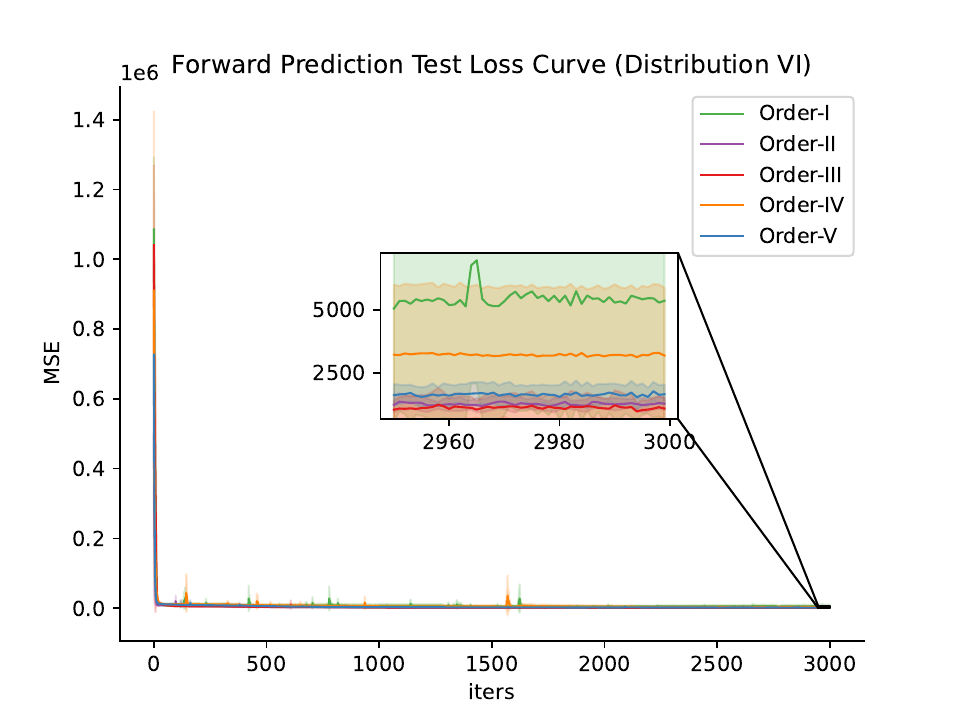}
        % \centerline{Distribution VI}
    \end{minipage}
    \begin{minipage}[t]{0.33\linewidth}
        \centering 
        \includegraphics[width=\textwidth]{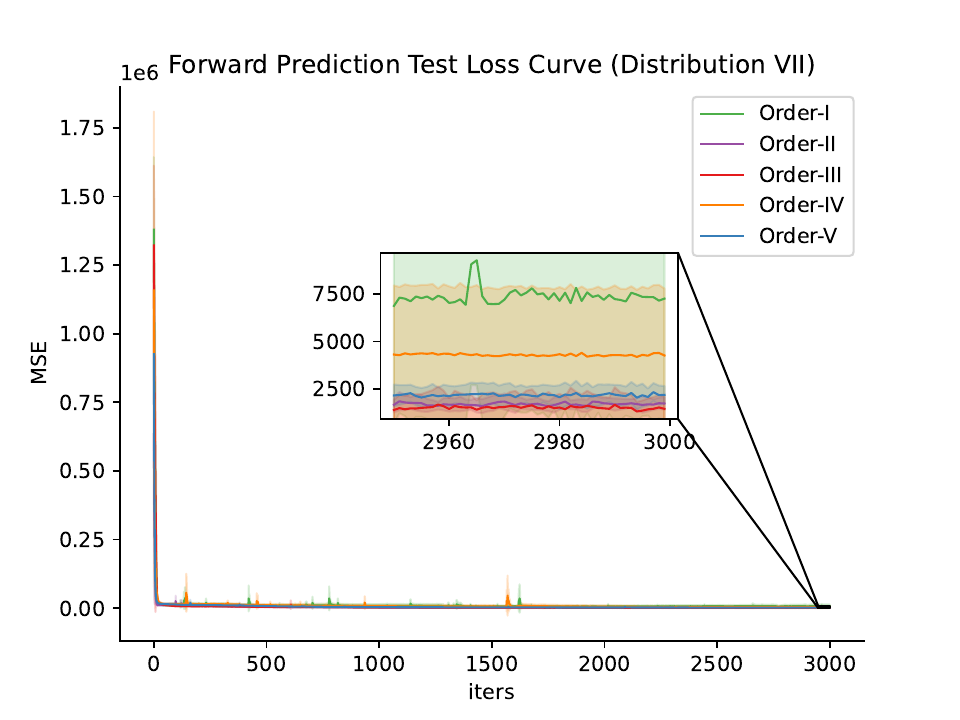}
        % \centerline{Distribution VII}
    \end{minipage}%
    \begin{minipage}[t]{0.33\linewidth}
        \centering
        \includegraphics[width=\textwidth]{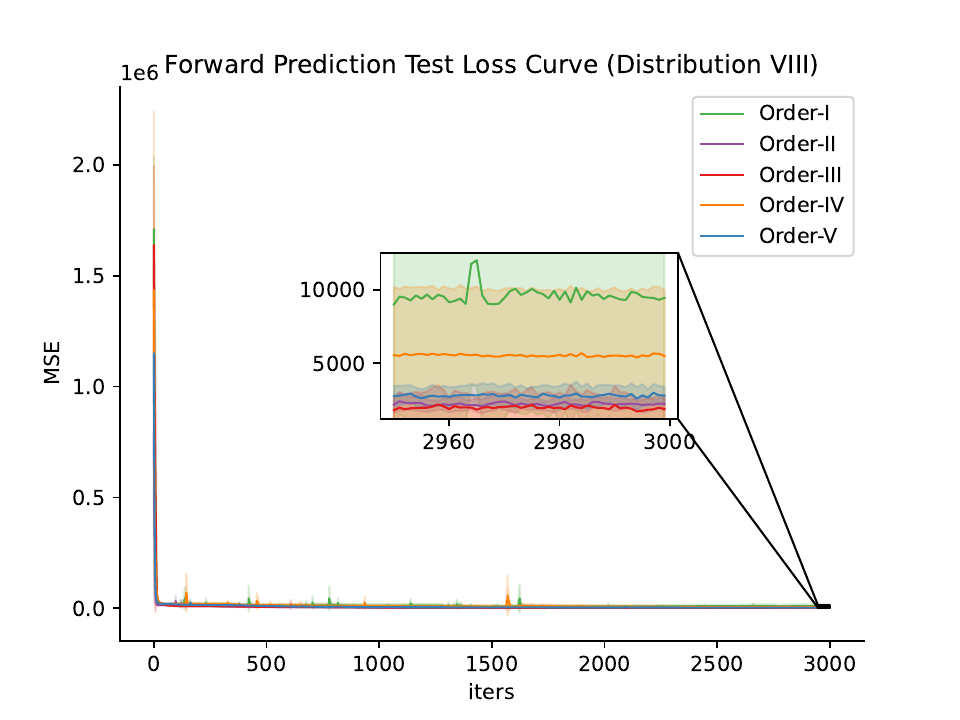}
        % \centerline{Distribution VIII}
    \end{minipage}

    \caption{Forward prediction of CINN under IID and OOD scenarios based on different causal topological orders.}
    \label{fig_app2}
\end{figure*}

\begin{figure*}[t]
    \begin{minipage}[t]{0.33\linewidth}
        \centering 
        \includegraphics[width=\textwidth]{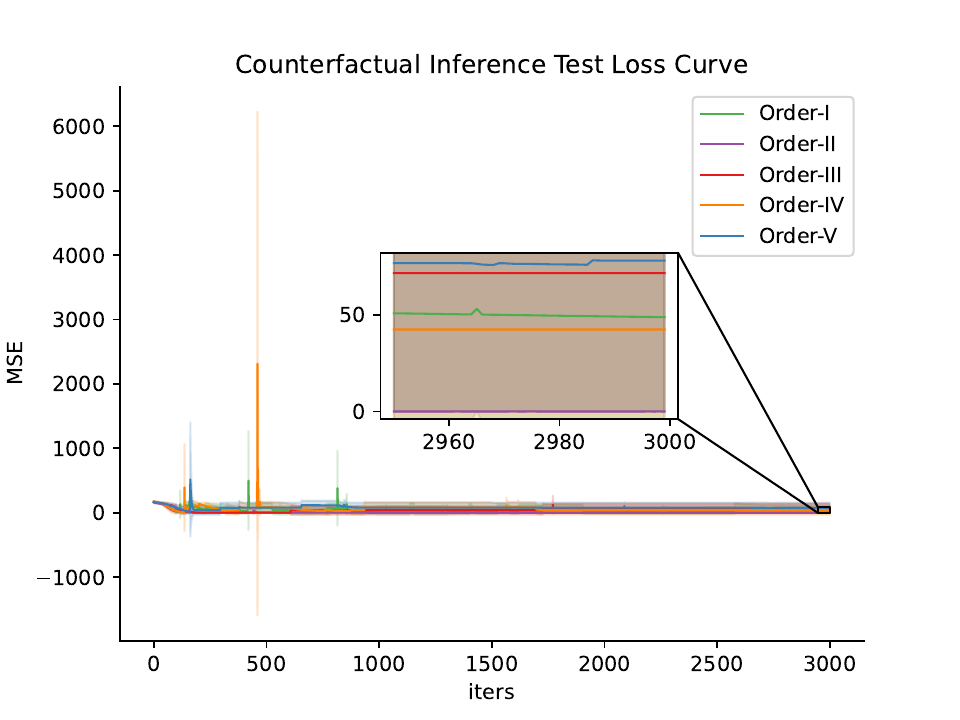}
        % \centerline{Distribution I}
    \end{minipage}%
    \begin{minipage}[t]{0.33\linewidth}
        \centering
        \includegraphics[width=\textwidth]{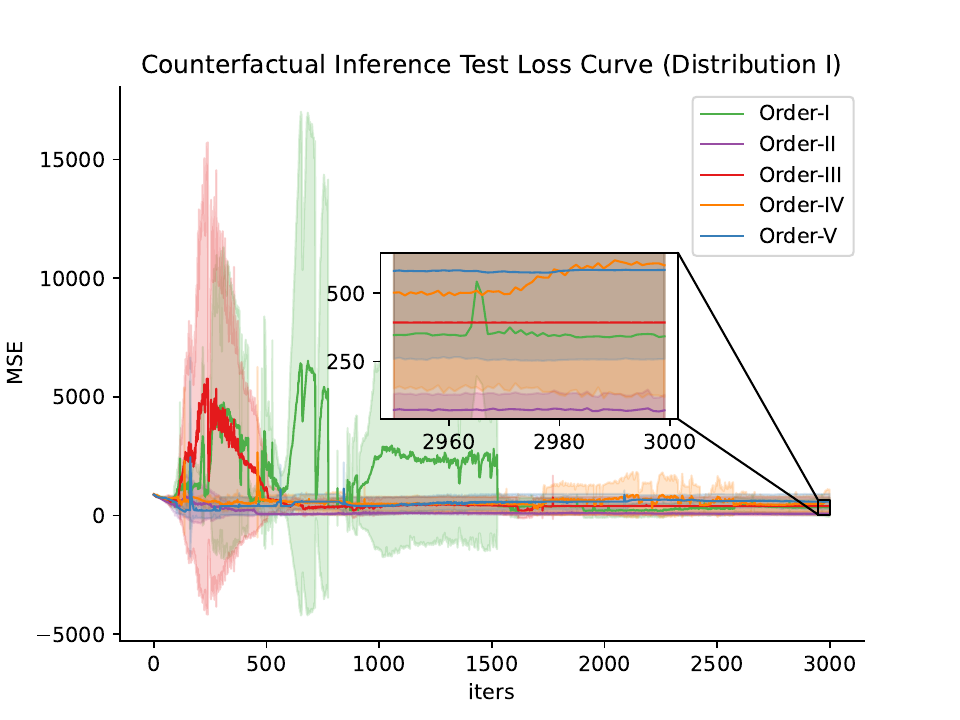}
        % \centerline{Distribution II}
    \end{minipage}
    \begin{minipage}[t]{0.33\linewidth}
        \centering 
        \includegraphics[width=\textwidth]{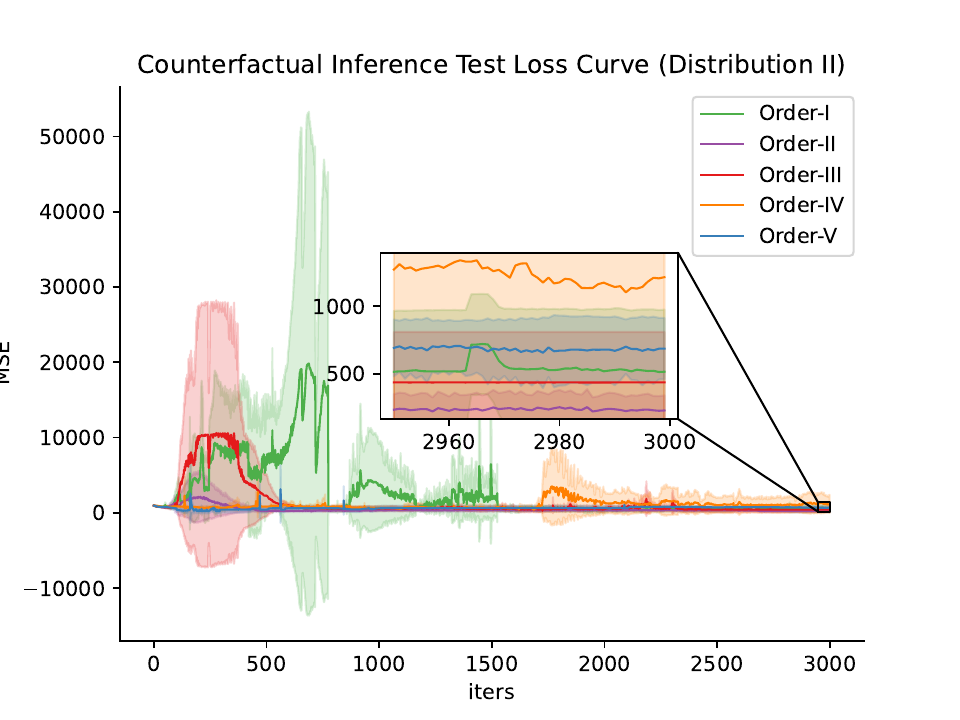}
        % \centerline{Distribution III}
    \end{minipage}%

    \begin{minipage}[t]{0.33\linewidth}
        \centering 
        \includegraphics[width=\textwidth]{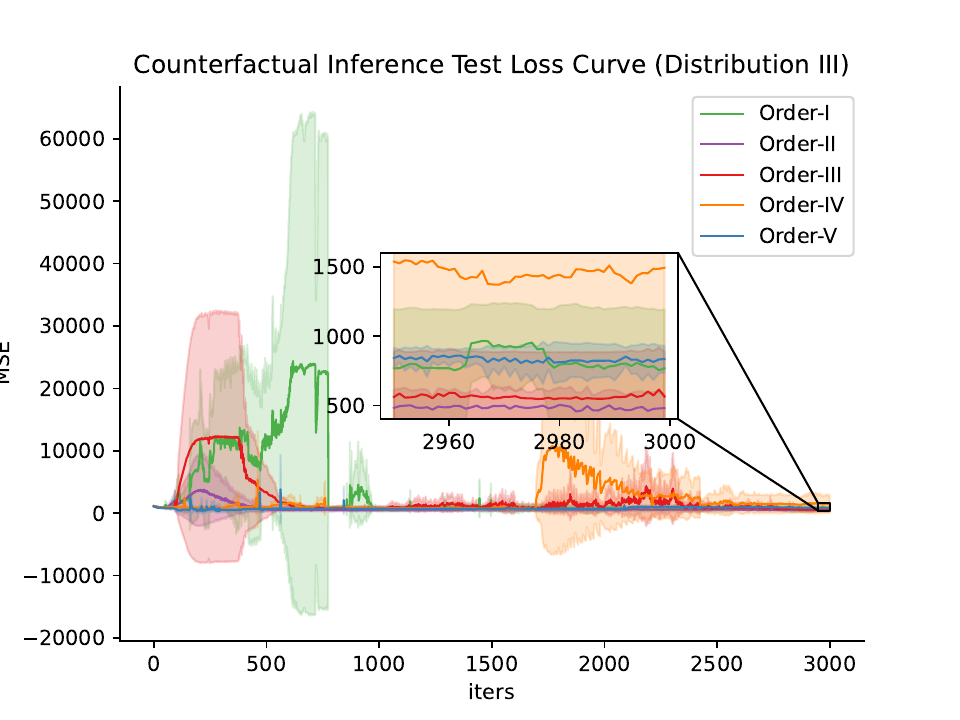}
        % \centerline{Distribution III}
    \end{minipage}%
    \begin{minipage}[t]{0.33\linewidth}
        \centering
        \includegraphics[width=\textwidth]{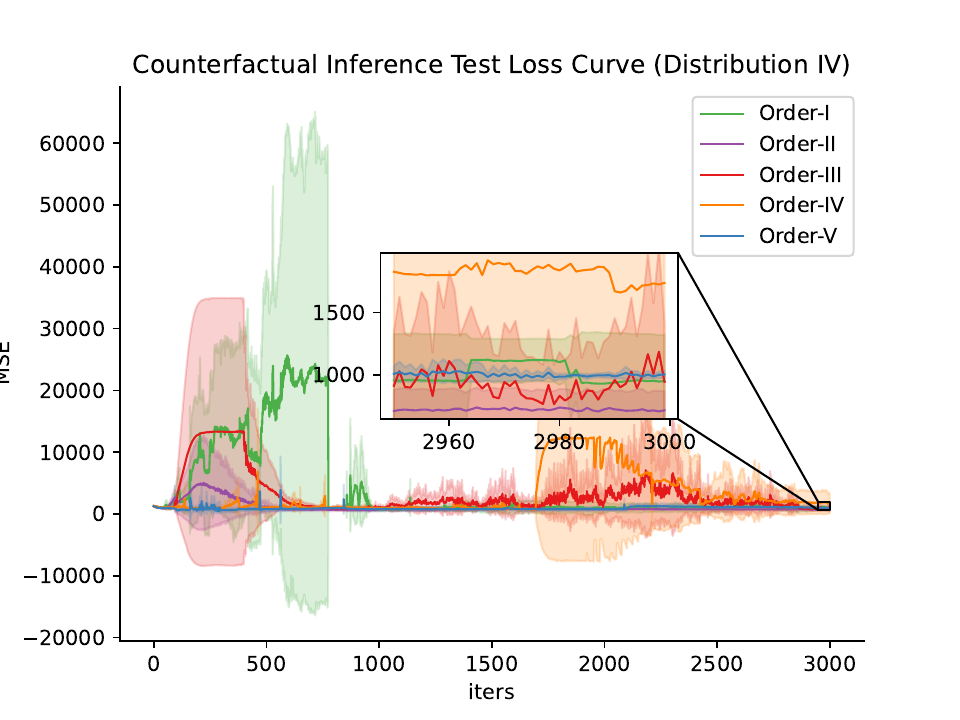}
        % \centerline{Distribution IV}
    \end{minipage}
    \begin{minipage}[t]{0.33\linewidth}
        \centering 
        \includegraphics[width=\textwidth]{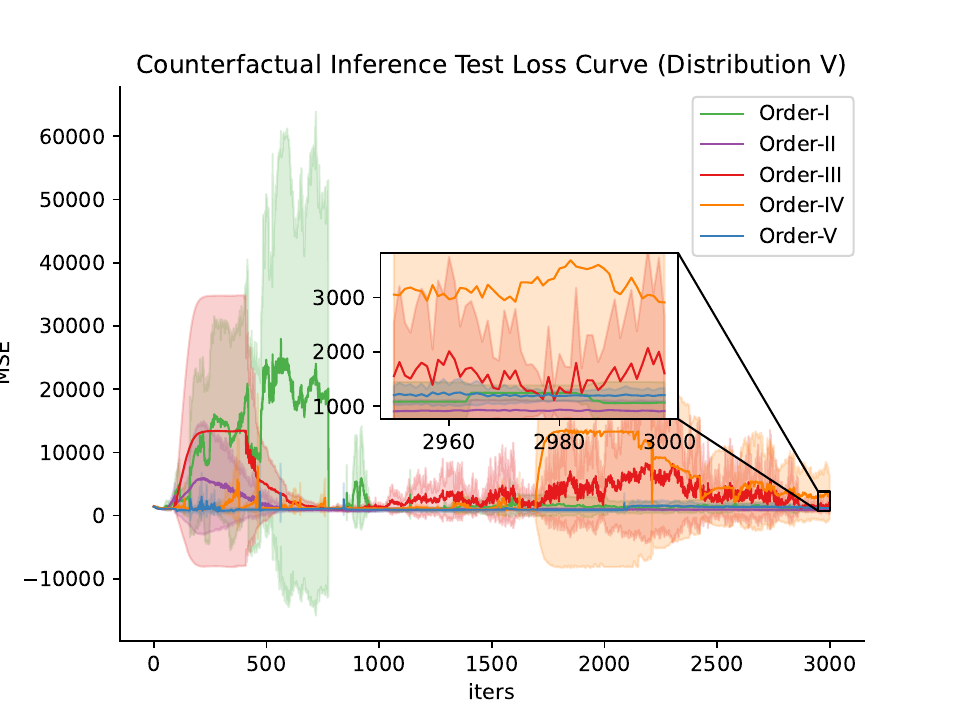}
        % \centerline{Distribution V}
    \end{minipage}%

    \begin{minipage}[t]{0.33\linewidth}
        \centering
        \includegraphics[width=\textwidth]{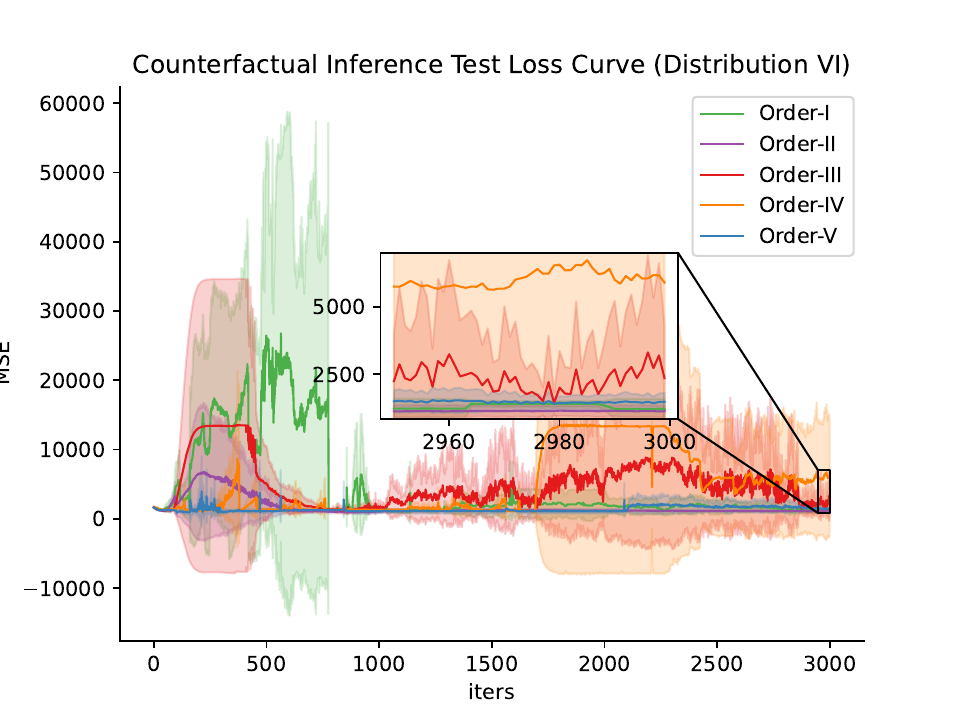}
        % \centerline{Distribution VI}
    \end{minipage}
    \begin{minipage}[t]{0.33\linewidth}
        \centering 
        \includegraphics[width=\textwidth]{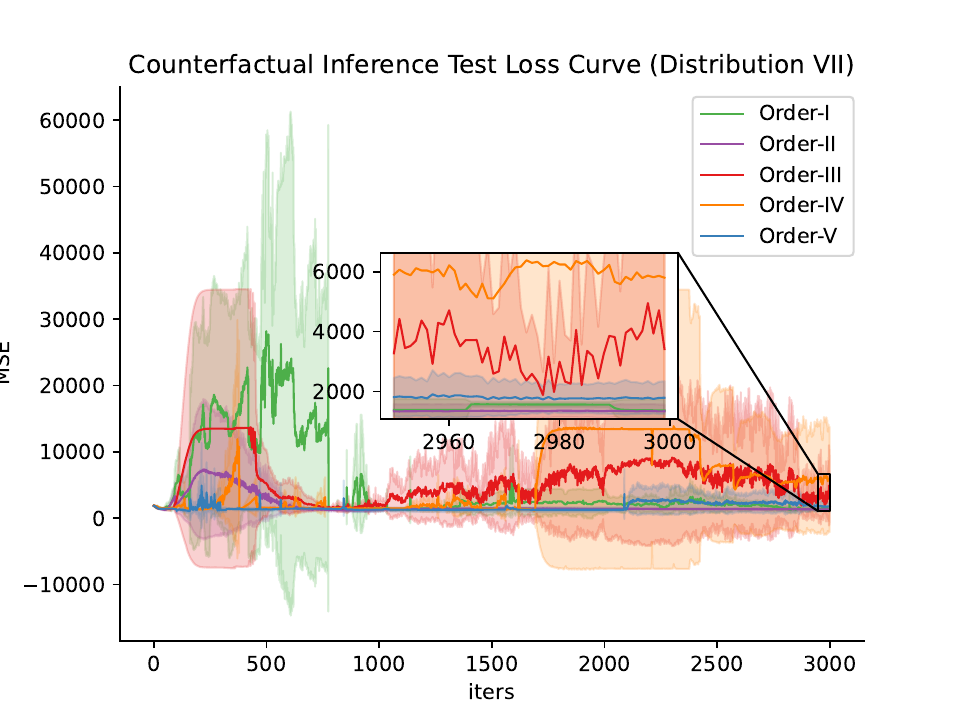}
        % \centerline{Distribution VII}
    \end{minipage}%
    \begin{minipage}[t]{0.33\linewidth}
        \centering
        \includegraphics[width=\textwidth]{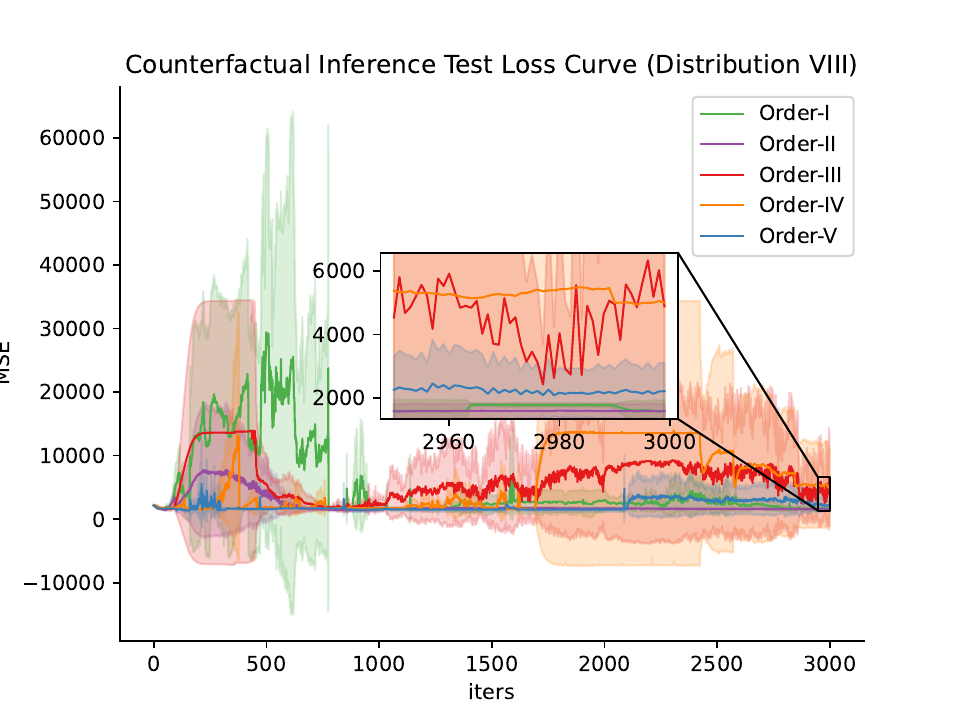}
        % \centerline{Distribution VIII}
    \end{minipage}

    \caption{Counterfactual inference of CINN under IID and OOD scenarios based on different causal orders.}
    \label{fig_app3}
\end{figure*}

\begin{equation}
	\begin{aligned}
    r^t_i=-\frac{\eta}{2}||s^{t+1}-F(s^t,a^t)||_2^2
	\end{aligned}
    \label{equation6}
\end{equation}

where $\eta>0$ is a scaling factor (default 1) and $F(s^t,a^t)$ is the estimated state transition function. Surprisingly, the state transition function $\hat{s}^{t+1}=F(s^t,a^t; W)$ exactly matches the forward prediction of CINN. Therefore, as shown in Figure \ref{fig3}, the forward prediction module, which is similar to the Markov decision process in RL, uses the forward prediction of CINN to output the intrinsic reward $r^t_i$.

\subsection{Introspection-driven Exploitation}

However, the RL undoubtedly introduces an element of inappropriate random exploration, potentially leading to unsafe or unstable outcomes \citep{garcia2015comprehensive,emerson2023offline}. We address mindless exploitation of the environment during agent-environment interaction through pre-action introspection. By analyzing the current and future states of T1D patients, the agent can actively infer the optimal action to achieve the desired state. The output of the counterfactual inference in CINN, $\hat{a}^t=I(s^t,s^{t+1})$, serves as the predicted estimate of the action $a^t=\pi(s^t;\psi)$ generated by the policy. Essentially, $\hat{a}^t$ reflects the agent's deliberation before taking the action $a^t$, and the primary goal of training is to synchronize the thought process with the action execution. Therefore, the policy is trained to optimize:

\begin{equation}
	\begin{aligned}
        \min_{\psi}L_I(I(s^t,s^{t+1}),\pi(s^t;\psi))
	\end{aligned}
    \label{equation8}
\end{equation}

where $L_I$ is the loss function that measures the discrepancy between the counterfactuals and the actual actions. The final optimization function is:

\begin{equation}
	\begin{aligned}
        \min_{\psi}[-\lambda\mathbb{E}_{\pi(s^t;\psi)}[\sum^t r^t]+(1-\beta)L_I+\beta L_F]
	\end{aligned}
    \label{equation9}
\end{equation}

where $0\le\beta\le 1$ is a scalar (default 0.5) that weights the counterfactual loss against the prediction loss, and $\lambda>0$ is a scalar (default 1) that weights the importance of the policy gradient loss against that of learning the introspection-driven loss.

\section{Experiments}

The effectiveness of CINN and our proposed introspective RL will be evaluated using the Diabetes Mellitus Metabolic Simulator for Research (DMMS.R) \citep{man2014uva}, a specialized computer application approved by the US FDA for conducting virtual clinical trials. This simulator provides real-world health records of patients with T1D and an interactive platform for managing BG levels. This section includes two different sets of experiments. First, supervised learning experiments are conducted using DMMS.R to validate CINN's forward prediction and counterfactual inference capabilities. Then, we perform BG control experiments to evaluate the safety and stability of our introspective RL.

\subsection{Experimental Setup\label{section51}}

\subsubsection{Supervised Learning Experiment Description} 

We collect the dataset of nine action distributions from DMMS.R, where the same patient receive different insulin doses while adhering to identical dietary standards daily. This dataset encapsulates patient state trajectories under different exogenous insulin dose distributions. To evaluate forward prediction accuracy under IID and OOD scenarios, we use intervention trajectories from one type of distribution for training and all nine types of distributions for testing.

For comparative analysis, we considered four basic forward prediction models: MLP \citep{taud2018multilayer}, RNN \citep{schuster1997bidirectional}, LSTM \citep{hochreiter1997long}, GRU \citep{cho2014learning}, and six state-of-the-art models: Transformer \citep{vaswani2017attention}, Autoformer \citep{wu2021autoformer}, Informer \citep{zhou2021informer}, FEDformer \citep{zhou2022fedformer}, Non-stationary Transformer \citep{liu2022non}, and DLinear \citep{zeng2023transformers} as baselines. Notably, the above ten baselines lack the intrinsic capacity for counterfactual inference, while CINN is the first model capable of performing both forward prediction and counterfactual inference within a unified network architecture. Therefore, we introduce a masking mechanism that adjusts inputs, allowing all of the baselines above to apparently perform forward prediction and counterfactual inference simultaneously. During forward prediction, a mask is applied to hide $s^{t+1}$ information, allowing baselines to receive the current state $s^t$ and action $a^t$ for predicting the subsequent state: $\hat{s}^{t+1},[\rm{mask}] = f(s^t, [\rm{mask}], a^t)$. During counterfactual inference, we mask the intervention information $a^t$ in the input so that the baselines receive only the current state $s^t$ and the desired state $\tilde{s}^{t+1}$ to perform counterfactual inference for the optimal intervention action: $[\rm{mask}], \hat{a}^{t} = f(s^t, \tilde{s}^{t+1}, \rm{mask})$.

\subsubsection{Reinforcement Learning Experiment Description} 

We treat DMMS.R as an interactive environment and use the adolescent patient parameters to develop the BG control policy. To complement the standard SAC, we integrate the pre-trained CINN as an introspective block to furnish prediction-driven and introspective-driven rewards for policy updates. In addition to our approach, we employ A2C \citep{mnih2016asynchronous}, PPO \citep{schulman2017proximal}, SAC \citep{haarnoja2018soft}, SAC+ICM \citep{pathakICMl17curiosity} as baselines.

\subsection{Evaluation Metrics}

We use the mean square error (MSE) to evaluate the accuracy of forward prediction and counterfactual inference, which is calculated as follows:

\begin{equation}
    \begin{split}
        \text{MSE}(s^{t+1},\hat{s}^{t+1})&=\frac{1}{n\times T}\sum_t(s^{t+1}-\hat{s}^{t+1})^2\\
        \text{MSE}(a^{t},\hat{a}^{t})&=\frac{1}{2\times T}\sum_t(a^{t}-\hat{a}^{t})^2
        \label{4_1}
    \end{split}
\end{equation}

Here, $\hat{s}^{t+1}\in\mathbb{R}^{n}$ represents the predicted future state in forward prediction, while $s^{t+1}$ denotes the corresponding ground-truth state with $T$ time steps for action trajectories. Additionally, $\hat{a}^{t}\in\mathbb{R}^{2}$ represents the action inferred through counterfactual inference, with $a^{t}$ denoting the corresponding ground-truth action. Note that in DMMS.R, $n=13$.

\subsection{Implementation details}

CINN and all of the above baselines are optimized using the Adam optimizer with an initial learning rate of $10^{-3}$. The batch size is set to 32. CINN is implemented in PyTorch and conducted on a single NVIDIA GeForce RTX 3090 24GB GPU. The code is available at \url{https://github.com/HITshenrj/CINN}.

\subsection{Experimental Results about Forward Prediction and Counterfactual Inference under IID Scenarios.
\label{section52}}

The results of the supervised learning experiments in IID scenarios are shown in Figure \ref{fig4}. It is evident that the ten baseline models exhibit robust performance in forward prediction. In addition, despite its reliance on simple FNNs, CINN demonstrates remarkable capabilities in forward prediction, contributing to stable training and convergence.

Concerning counterfactual inference, the four basic models exhibit smaller training and testing MSEs compared to the Transformer-based models. However, we note that a possible reason for this performance gap is overfitting. In addition, the six state-of-the-art models show significant fluctuations in loss metrics. In contrast, CINN maintains relative stability during counterfactual inference. This suggests that CINN can analyze the relationships between the states of patients with T1D, exogenous insulin doses, and carbohydrate intake through its network structure, which is constructed based on causal relationships rather than correlation. As a result, CINN achieves accurate forward prediction and counterfactual inference.

\subsection{Generalization Experimental Results about Forward Prediction and Counterfactual Inference\label{section53}}

As shown in Figure \ref{fig5}, our generalization experiments show that MLP, Non-stationary Transformer, DLinear, and CINN exhibit superior forward prediction performance in OOD scenarios. However, regarding counterfactual inference, MLP and Non-stationary Transformer fail to deliver satisfactory results and exhibit extreme instability, largely due to their heavy reliance on correlation. As a result, they have difficulty capturing the underlying data generation process and struggle to adapt to varying data distributions. Although DLinear outperforms other baselines in both forward prediction and counterfactual inference, it still exhibits a noticeable gap with CINN.

CINN shows robust performance in both forward prediction and counterfactual inference, maintaining minimal losses regardless of the action distribution. This effectiveness of CINN can be attributed to its interpretable construction generation based on causal topological order, which ensures consistency despite shifts in data distribution. In addition, CINN's network structure is adept at bidirectional computation, allowing it to perform forward prediction and counterfactual inference simultaneously.

\begin{figure*}[t]
    \begin{minipage}[t]{0.33\linewidth}
        \centering 
        \includegraphics[width=\textwidth]{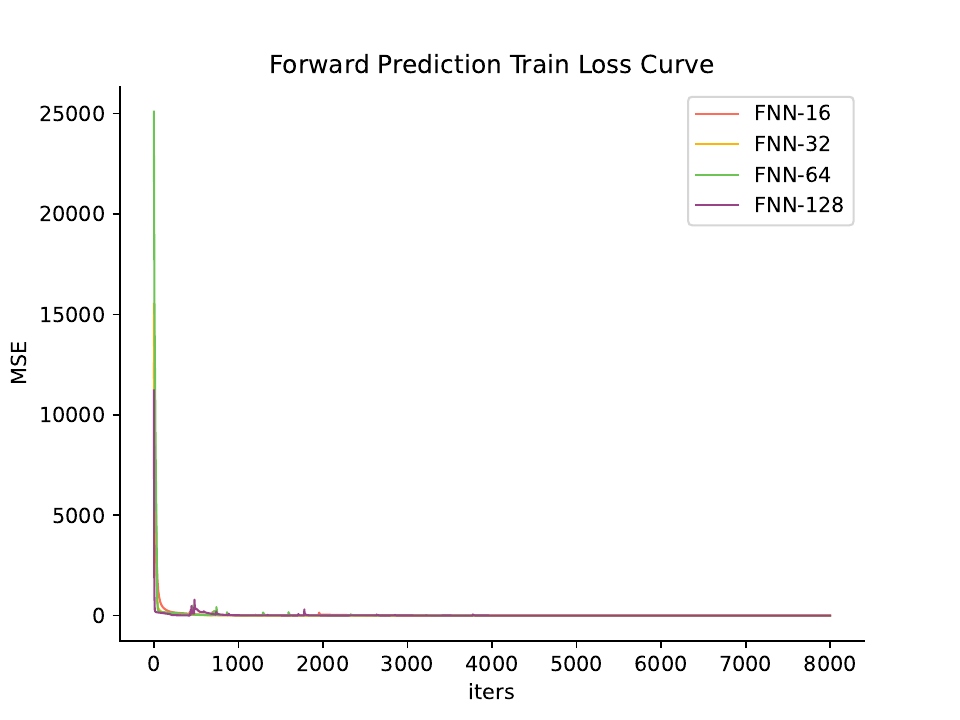}
        % \centerline{Distribution I}
    \end{minipage}%
    \begin{minipage}[t]{0.33\linewidth}
        \centering
        \includegraphics[width=\textwidth]{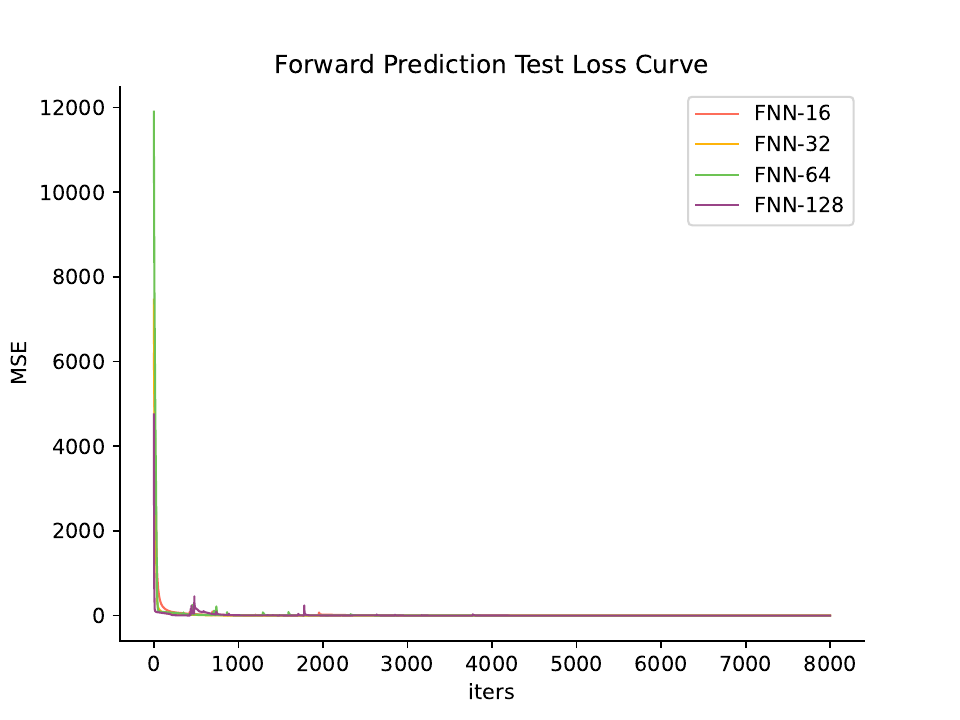}
        % \centerline{Distribution II}
    \end{minipage}
    \begin{minipage}[t]{0.33\linewidth}
        \centering 
        \includegraphics[width=\textwidth]{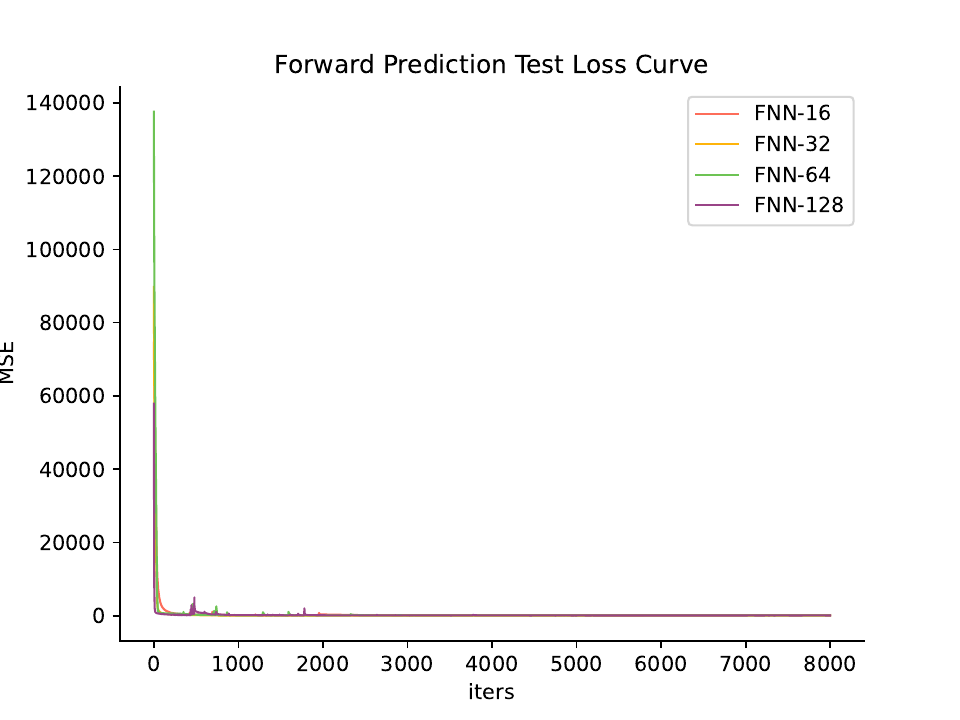}
        % \centerline{Distribution III}
    \end{minipage}%

    \begin{minipage}[t]{0.33\linewidth}
        \centering 
        \includegraphics[width=\textwidth]{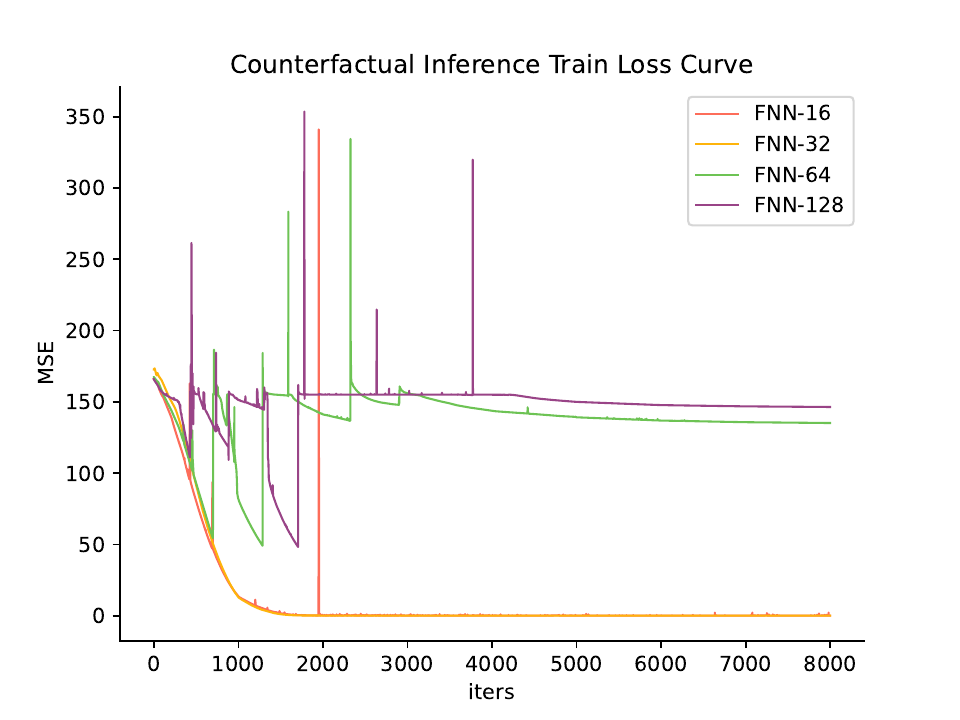}
        % \centerline{Distribution III}
    \end{minipage}%
    \begin{minipage}[t]{0.33\linewidth}
        \centering
        \includegraphics[width=\textwidth]{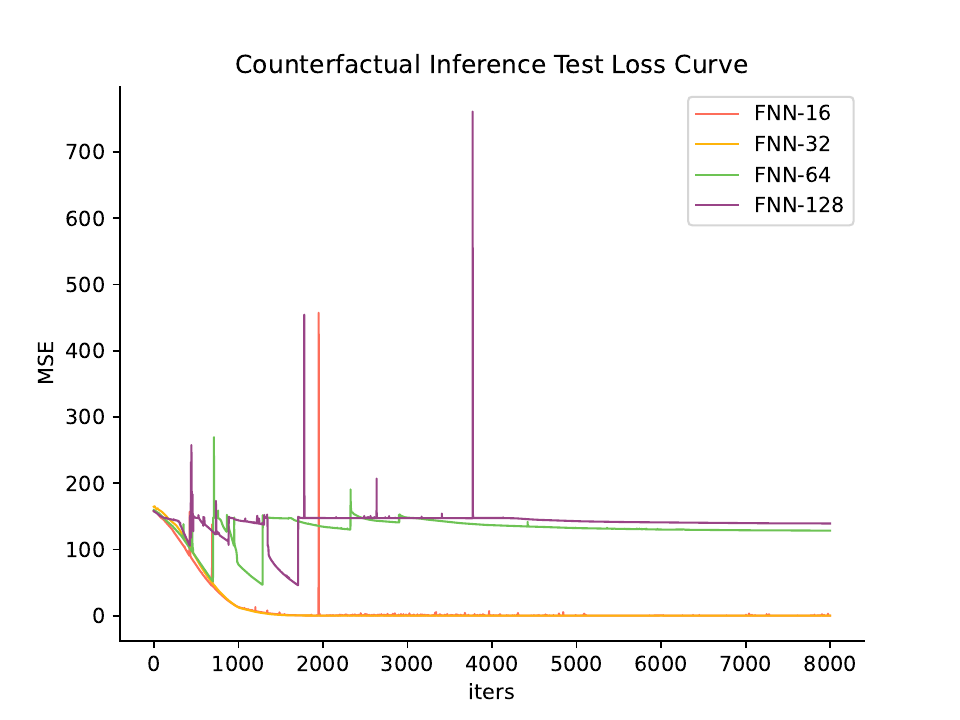}
        % \centerline{Distribution IV}
    \end{minipage}
    \begin{minipage}[t]{0.33\linewidth}
        \centering 
        \includegraphics[width=\textwidth]{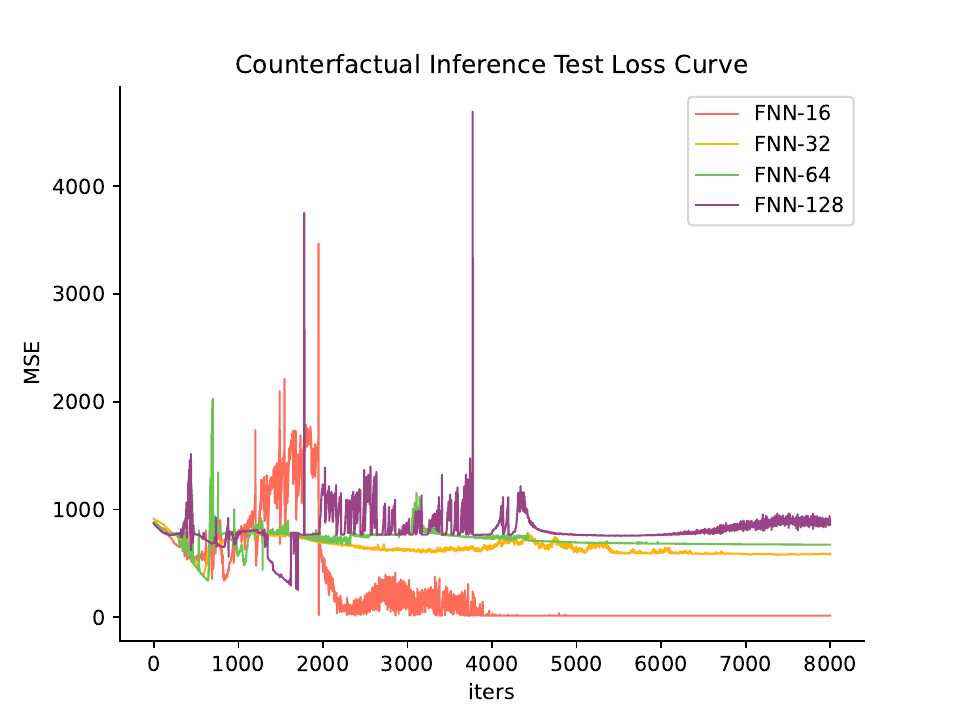}
        % \centerline{Distribution V}
    \end{minipage}%

    \caption{Ablation Experiments on Parameter Scale of FNN.}
    \label{fig_app4}
\end{figure*}

\subsection{Ablation Experiments For CINN}

In the following experiments, we first conduct an ablation study on the causal topological orders to assess its significant impact on CINN. Additionally, we perform another ablation study on the CINN parameter scales to elucidate hyperparameters' role in influencing forward prediction and counterfactual inference.

\subsubsection{Ablation Experiments on Topological Orders}

To showcase the pivotal role of causal graphs in CINN, we carry out ablation experiments on the causal topological order used to construct the network architecture. The original causal topological order (Order-I) is based on the dynamic equation in DMMS.R. We generate four other orders (Order-II to Order-V) by randomly transforming the original causal topological order.

As depicted in Figure \ref{fig_app2}, the experimental results indicate that all modified models exhibit favorable predictive performance during the forward prediction process. With the gradual escalation of the differences in the data distributions (ranging from Distribution-I to Distribution-VIII), the loss incurred by each model also increases, but remains within an acceptable range. The counterfactual inference experiments of CINN under OOD scenarios with different causal topological orders, are illustrated in Figure \ref{fig_app3}. Order-I yields optimal test results, showcasing remarkable stability. This observation underscores the imperative of structuring neural networks strictly according to causal relationships, thereby integrating fundamental causal knowledge into the model. On the contrary, variant models beyond the Order-I and Order-II exhibit substantial fluctuations, likely stemming from the deviation of the causal topological orders derived via randomly transformation from the true causal relationships. This mismatch leads to neural network configurations that fail to accurately represent the underlying causal relationships among variables. Thus, CINN relies on adherence to the causal order corresponding to the causal graph, with the incorporation of prior causal graphs significantly improving the model's performance.

\subsubsection{Ablation Experiments on Parameter Scales}

We perform ablation experiments on the parameter size of the FNN within the CINN. In this ablation analysis, we keep the FNN structure constant while varying only the size of the hidden dimension. The hidden layer sizes are set to 16 (default), 32, 64, and 128 dimensions, respectively. The results are shown in Figure \ref{fig_app4}.

Under the experimental conditions of this study, we consistently observe that a smaller network model with fewer parameters consistently yields better results in terms of MSE, regardless of whether under IID or OOD scenarios for forward prediction or backward inference. This observation suggests that networks with larger parameter sizes are susceptible to overfitting during forward prediction, causing the learned relationships to be biased toward misleading correlations. Furthermore, using models based on correlations for counterfactual inference generally leads to poorer experimental results. These findings suggest that a larger network size does not necessarily translate into better performance. Larger parameter sizes may emphasize correlational relationships to achieve lower losses, potentially compromising the effectiveness of forward prediction and counterfactual inference. Notably, the CINN based on FNN-16 retains causal generalization even across different action distributions, whereas other variant models cannot effectively transfer knowledge across different datasets.

\subsection{Stable and Safe BG Control\label{section54}}

As shown in Figure \ref{fig6}, we are delving into continuous BG control experiments. We aim to regulate patients' BG levels within the target range of 70-180 mg/dl by adjusting their exogenous insulin injection doses. We integrate the pre-trained CINN into SAC because SAC has the most significant room for improvement among A2C, SAC, and PPO. In addition, we view the CINN as a frozen introspective block that guides policy updates. This introspective approach involves performing counterfactual inference before executing each action prescribed by the policy. This step aims to reduce futile trials and potential errors during training, ultimately leading to a more stable and safer BG control.

The cumulative reward curves highlight the inconsistency in reward increments for PPO and A2C as the training epochs progress. These two methods often result in BG levels that predominantly exceed the upper bound of the target range. This observation underscores the limitations of these conventional RL policies, mainly due to their limited exploration capabilities, which can lead to trapping in local optima. Although SAC significantly improves cumulative reward, we note the occurrence of low BG levels due to excessive exogenous insulin dosing, which poses significant risks to patients with T1D. Further results of the SAC+ICM curve show persistent fluctuations between BG decreases and increases, failing to achieve a stable BG phase. We contribute such persistent instability to the unsafe trials from curiosity-based rewards.

Because of SAC's most significant room for improvement, we integrate the pre-trained CINN into SAC, effectively infusing prior causal knowledge into the RL paradigm. We find that this RL+CINN paradigm can consistently maintain T1D patients' BG levels within the target range while ensuring both stability and safety. This integration allows the RL policy to capture the underlying evolutionary mechanism of patients' states concerning external interventions, thereby achieving effective BG control. Furthermore, introducing introspective rewards enhances RL's safe exploration capacity, accelerating the control policy's convergence and resulting in higher cumulative rewards with improved stability.

\section{Discussion}

This study presents, for the first time, an introspective RL algorithm with pre-trained CINN for stable and safe BG control. Key findings include: (1) the pre-trained CINN demonstrates significantly superior performance compared to baselines on forward prediction and counterfactual inference; (2) through ablation experiments on topological order, the topological order of the true causal graphs significantly improves the performance of the CINN and ensures the interpretability of the network structure construction; (3) ablation experiments on parameter scales confirm the lightweight nature of the CINN, which can efficiently serve as an introspection block to guide policy updates; (4) the use of RL+CINN exhibits optimal model performance, achieving stable and safe BG control. In this section, we will focus on the reasons why safe and stable BG control can be achieved after RL+CINN.

Just as the teacher network guides the student network in knowledge distillation \citep{shi2021covid,pham2021cnn}, our pre-trained CINN also guides policy updating within RL. What makes it different, however, is that the CINN also prompts the policy to engage in introspection, often compared to perception, reasoning, memory, and testimony as sources of knowledge \citep{perner2007introspection}. First, with the help of the pre-trained CINN, the RL policy can receive a prediction-driven reward based on predicting the future from the current state and action. If the predicted future state $\hat{s}^{t+1}$ deviates significantly from the actual $s^{t+1}$, indicating a potentially unsafe action by the policy, this reward encourages the policy to recognize and subsequently generate a safer action in the following process. Second, the RL policy can also receive an introspection reward based on counterfactual inference actions from the current and future states. This reward encourages the policy to take more rational actions. With the combined effect of these two additional rewards, our RL+CINN paradigm achieves more stable and safer performance in BG control.

\begin{figure}[t]
    \begin{minipage}[t]{0.5\linewidth}
        \centering 
        \small
        \includegraphics[width=\textwidth]{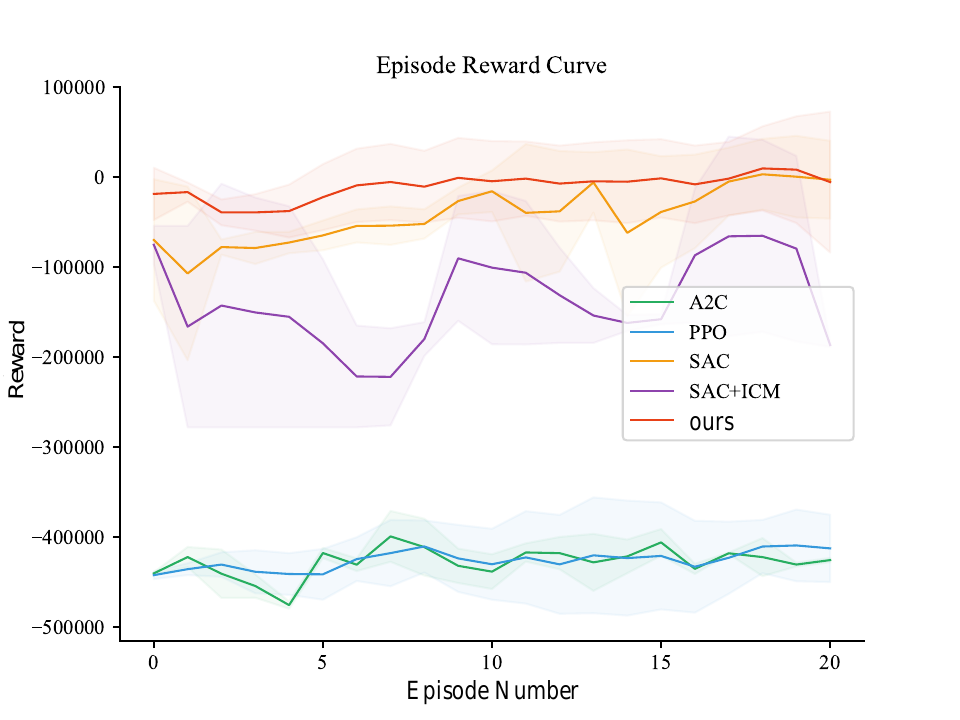}
        \centerline{(a) Cumulative Reward Curve}
    \end{minipage}%
    \begin{minipage}[t]{0.5\linewidth}
        \centering
        \small
        \includegraphics[width=\textwidth]{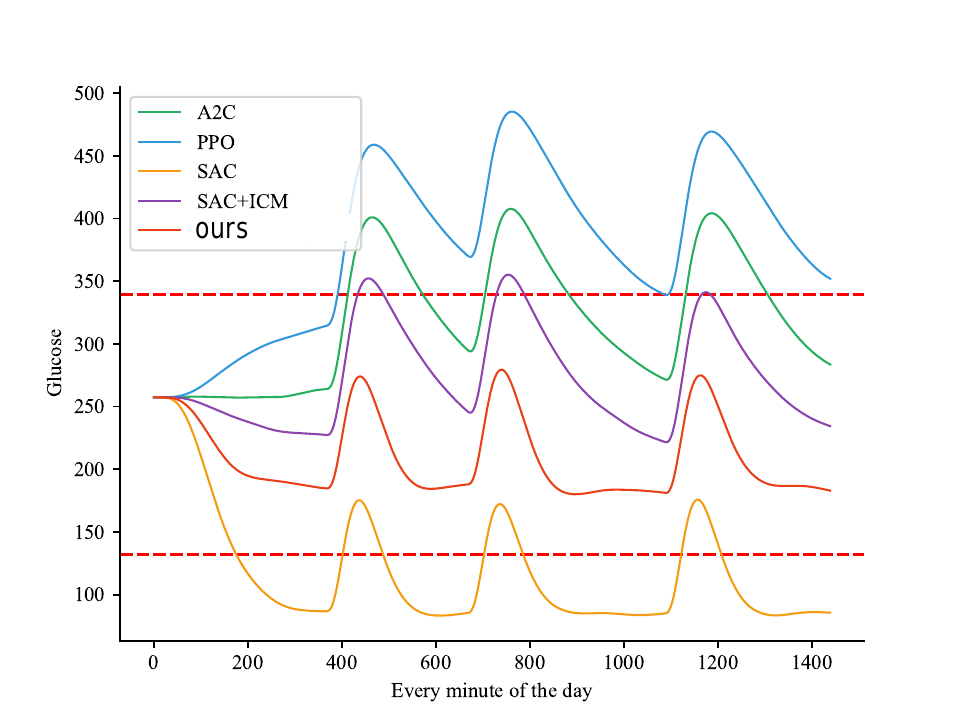}
        \centerline{(b) Blood Glucose Curve}
    \end{minipage}
    \caption{The effect of CINN and the comparison model in the BG control.}
    \label{fig6}
\end{figure}

\section{Conclusion}

This paper presents a BG control method using pre-trained CINN, a bidirectional neural network that integrates forward prediction and counterfactual inference. CINN employs affine coupling layers and orthogonal weight normalization for forward prediction of future states under external interventions and counterfactual introspection of expected actions. Experimental results show that pre-trained CINN exhibits high accuracy and generalizability characteristics, outperforming other baselines under both IID and OOD scenarios. By incorporating pre-trained CINN as a frozen introspective block to guide RL updates, the BG control experiments show exceptional stability and safety. This integration promotes counterfactual thinking before each policy action, thereby reducing uncertain trials. In addition, experiments show that introspective RL converges faster, resulting in higher and more stable cumulative rewards. In BG control tasks, the SAC+CINN achieves optimal performance, maintaining BG levels within the target range and minimizing the risk of hypoglycemia and hyperglycemia.

Despite the remarkable performance of introspective RL, it still faces certain limitations. One limitation is its dependence on the accuracy of the pre-trained CINN, as an inaccurate CINN could render introspective RL ineffective. In addition, introspective RL does not account for situations where users do not eat on time, which is a common out-of-distribution (OOD) generalization scenario. For future research, we intend to explore a more comprehensive BG control model based on the random meal times of patients according to DMMS.R, which aligns more closely with real-world scenarios.

\section{Acknowledgments}

This study was supported in part by a grant from the National Key Research and Development Program of China [2021ZD0110900] and the National Natural Science Foundation of China [62006063].

\bibliographystyle{unsrtnat}
\bibliography{references}

\end{document}